\documentclass[10pt,twocolumn,letterpaper]{article}

\usepackage{cvpr}            

\usepackage{titletoc}
\usepackage{amsmath}
\usepackage{amssymb}
\usepackage{amsthm}
\usepackage{mathtools}
\usepackage{bm}
\usepackage{physics}
\usepackage{nicefrac}
\usepackage{amsfonts}
\newtheorem{assumption}{Assumption}
\newtheorem{lemma}{Lemma}

\usepackage{booktabs}
\usepackage{array}
\usepackage{multirow}
\usepackage{graphicx}
\usepackage{graphics}
\usepackage{wrapfig}
\usepackage{xcolor}
\usepackage{microtype}
\usepackage{booktabs}
\usepackage{multirow}
\usepackage{tabularx}
\usepackage[ruled,vlined]{algorithm2e}
\usepackage[table]{xcolor}
\definecolor{ourblue}{RGB}{220,235,252}
\usepackage{makecell}
\usepackage[accsupp]{axessibility}  
\newtheorem{theorem}{Theorem}

\usepackage{amsmath,amsfonts,bm}

\def\eqref#1{equation~\ref{#1}}

\def\1{\bm{1}}

\DeclareMathAlphabet{\mathsfit}{\encodingdefault}{\sfdefault}{m}{sl}
\SetMathAlphabet{\mathsfit}{bold}{\encodingdefault}{\sfdefault}{bx}{n}

\DeclareMathOperator*{\argmax}{arg\,max}

\definecolor{cvprblue}{rgb}{0.21,0.49,0.74}
\usepackage[
  pagebackref,
  breaklinks,
  colorlinks,
  allcolors=cvprblue
]{hyperref}

\def\paperID{*****}
\def\confName{CVPR}
\def\confYear{2026}

\title{FRAME: Forensic Routing and Adaptive Multi-path Evidence Fusion for Image Manipulation Detection}

\author{
Kaixiang Zhao\textsuperscript{1} \quad
Tianrun Yu\textsuperscript{1} \quad
Aoxu Zhang\textsuperscript{2} \quad
Junhao Su\textsuperscript{1} \\
Porter Jenkins\textsuperscript{1} \quad
Amanda Hughes\textsuperscript{1}\thanks{Corresponding author: \texttt{amanda\_hughes@byu.edu}}\\[2pt]
\textsuperscript{1}Brigham Young University \quad
\textsuperscript{2}Rutgers University\\[2pt]
{\tt\small \{kzhao2, tianruny, sjh666, pjenkins, amanda\_hughes\}@byu.edu}\\
{\tt\small aoxu.zhang@rutgers.edu}
}

\begin{document}
\maketitle

\begin{abstract}
The proliferation of sophisticated image editing tools and generative artificial intelligence models has made verifying the authenticity of digital images increasingly challenging, with important implications for journalism, forensic analysis, and public trust.
Although numerous forensic algorithms, ranging from handcrafted methods to deep learning-based detectors, have been developed for manipulation detection, individual methods often suffer from limited robustness, fragmented evidence, or weak generalization across manipulation types and image conditions.
To address these limitations, we present \textbf{FRAME}, a method for \textbf{F}orensic \textbf{R}outing and \textbf{A}daptive \textbf{M}ulti-path \textbf{E}vidence fusion for image manipulation detection.
FRAME organizes diverse forensic algorithms into a multi-path analysis space, adaptively selects informative forensic paths for each input image, and fuses complementary evidence to improve detection and localization performance.
By moving beyond single-method analysis and fixed fusion strategies, FRAME provides a more robust and flexible approach to image forensic reasoning while preserving interpretable forensic cues from multiple evidence sources.
Experimental results demonstrate the effectiveness of FRAME across diverse manipulation scenarios.
Code is available at \href{https://github.com/kzhao5/FRAME}{https://github.com/kzhao5/FRAME}.
\end{abstract}


\section{Introduction}

Digital images are now central to communication, journalism, science, entertainment, and everyday decision making. They support reporting, documentation, medical analysis, online commerce, public safety, and social interaction. Their wide availability and low cost make them useful in many real-world applications. As a result, images are often treated as trustworthy evidence in both public and private settings.

However, this dependence on digital images also creates a serious security and trust problem. Modern editing tools make it easy to alter images while preserving a realistic appearance, which makes manipulation difficult to detect by visual inspection alone \cite{guillaro2023trufor, verdoliva2020media}. This threat matters because forged images can spread misinformation, damage reputations, distort public opinion, weaken scientific or legal evidence, and reduce trust in digital media \cite{verdoliva2020media}. For example, a manipulated image can be used to mislead news audiences, falsify visual evidence in an investigation, or create false medical or scientific records. These risks make reliable image manipulation detection an urgent need.

Many researchers have studied this problem. Traditional forensic methods analyze artifacts such as JPEG inconsistencies \cite{fridrich2009digital}, sensor noise \cite{chen2008determining}, or color filter array traces \cite{ferrara2012image}. These methods are interpretable, but each method usually captures only one type of evidence, so its output can be incomplete or unstable. Deep learning methods have improved performance across broader manipulation types and often combine multiple cues within one model \cite{bayar2016deep, guillaro2023trufor}. However, these models are often hard to interpret, and their decisions can be difficult to verify. Some recent work has explored fusion across multiple forensic signals, but most existing systems still rely on fixed combinations or fixed routing strategies. As a result, existing exploration in developing adaptive multi-cue forensic analysis remains nascent.

This problem remains difficult for three main reasons. First, different manipulations leave different forensic traces, so no single detector is reliable across all cases. Second, outputs from multiple forensic methods are often noisy, fragmented, or inconsistent, which makes it hard to combine them into one coherent prediction. Third, fixed fusion rules cannot adapt to differences in image content, compression history, and manipulation type. These challenges are not fully addressed by prior methods, and they limit both robustness and interpretability in practical forensic analysis.

To address these challenges, we introduce FRAME. The key idea is to treat different combinations of forensic tools as candidate analysis paths rather than forcing all tools to contribute equally. FRAME first organizes heterogeneous forensic algorithms into a supernet-inspired module pool. It then evaluates candidate paths and selects the most useful ones for the input image and its suspected manipulation context. This adaptive routing step addresses the problem that no single detector works well in every case. Next, FRAME fuses the evidence from the selected paths into a unified prediction map, which reduces noise and preserves complementary forensic cues. This fusion step addresses the problem of fragmented and inconsistent outputs. Finally, we provide a theoretical view of context-adaptive path selection and identify conditions under which an accurate selector can improve over uniform fusion and the best fixed single-algorithm baseline. In this paper, we focus on this computational core of adaptive path selection and evidence fusion.

\noindent The primary contributions of this work are:
\begin{itemize}
    \item We propose FRAME, an adaptive framework that routes and fuses multiple forensic cues for image manipulation detection.
    \item We introduce a supernet-inspired formulation of forensic analysis paths over heterogeneous forensic algorithms.
    \item We provide a theoretical analysis of context-adaptive path selection and identify conditions under which it improves over fixed baselines.
    \item We design a fusion mechanism that produces a unified prediction while preserving the contribution of individual forensic methods.
\end{itemize}
\section{Preliminaries and Related Work}
\label{sec:related}

\noindent\textbf{Notations and Problem Setup.}
Let \(I\) denote an input digital image, and let \(\mathcal{A} = \{A_1, A_2, \dots, A_N\}\) denote a set of \(N\) forensic algorithms used as building blocks in FRAME. These algorithms may include methods for compression analysis, sensor noise analysis, CFA artifacts, and copy-move detection, among others, many of which are available through existing forensic toolkits such as \textit{pyIFD}. Applying algorithm \(A_i \in \mathcal{A}\) to image \(I\) produces an output \(O_i = A_i(I)\), where \(O_i\) may take the form of a heatmap, mask, score, or other forensic signal.

We define a path \(P_j\) as a selected subset or ordered combination of algorithms from \(\mathcal{A}\) used to analyze a given image. The overall method can be viewed as a supernet-inspired structure \(\mathcal{S}\) that contains many possible analysis paths. Given an image \(I\), a suspected manipulation type \(M_t\), and optionally additional contextual information \(C\), the goal is to select an effective path \(P_j^*\) and fuse its outputs into a final result \(F\). This work focuses not on retraining the underlying forensic algorithms, but on adaptively selecting, and organizing their outputs within a unified analysis pipeline. This setup provides a common vocabulary for the algorithmic formulation in Section~\ref{sec:methodology} and the theoretical abstraction in Section~\ref{sec:theory}.

\noindent\textbf{Supernet-Inspired Formulation.}
The term supernet originates in the field of Neural Architecture Search (NAS) \cite{pham2018efficient, cai2018proxylessnas}, where a large neural network architecture contains many candidate subnetworks and search methods are used to identify effective task-specific configurations. We adopt this idea at a conceptual level. Instead of treating neural network layers as searchable components, we treat individual forensic algorithms as modular units that can be combined into alternative analysis paths.

This formulation allows FRAME to adapt its evidence integration strategy to the input image rather than applying all forensic tools uniformly. The key idea is to support adaptive path selection across multiple complementary forensic signals, followed by fusion of the selected outputs. Because FRAME operates over existing forensic methods rather than learned subnetworks in the usual NAS sense, the supernet terminology is best understood here as an organizing abstraction for adaptive multi-path forensic analysis.

\noindent\textbf{Image Manipulation Detection.}
Image manipulation detection has been studied through both traditional forensic methods and more recent learning-based approaches. Traditional methods typically examine artifacts introduced during image creation or editing. These include approaches such as Error Level Analysis (ELA) and Discrete Cosine Transform (DCT) coefficient analysis, which identify irregularities in JPEG compression patterns \cite{fridrich2009digital}. Noise-based techniques, including Photo Response Non-Uniformity (PRNU), use camera sensor noise patterns to detect inconsistencies that may indicate tampering \cite{chen2008determining}. Other methods focus on artifacts associated with Color Filter Array (CFA) interpolation \cite{ferrara2012image}, duplicated regions in copy-move forgeries \cite{christlein2012evaluation}, or inconsistencies in lighting and shadows \cite{kee2014exposing}. Libraries such as \textit{pyIFD} bring many of these algorithms together within a common toolkit. These methods are based on well-understood principles and can be effective for specific, known forms of manipulation. However, they also have important limitations. Their outputs, frequently in the form of heatmaps or binary masks, can be noisy, ambiguous, and difficult to interpret without significant expertise. Synthesizing evidence from multiple disparate traditional methods into a coherent assessment can also be challenging. In addition, these methods often struggle with novel manipulation techniques and complex post-processing operations. Support is typically limited to individual outputs, with little integration across methods.

Learning-based approaches, particularly Convolutional Neural Networks (CNNs), have improved performance across a broader range of manipulation types. Early deep learning models were often trained to detect specific manipulation types \cite{bayar2016deep}, while later work moved toward more general detection capabilities \cite{bayar2018constrained}. More recent architectures incorporate specialized layers to capture subtle forensic traces or use multi-stream and fusion strategies to combine evidence from different sources or modalities. For example, MMFusion combines outputs from various forensic filters \cite{ahmad2023mmfusion}, TruFor uses a transformer to fuse RGB information with noise fingerprints \cite{guillaro2023trufor}, and OMG-Fuser incorporates object guidance \cite{karageorgiou2024omgfuser}. While these deep learning models have achieved state-of-the-art accuracy on many benchmarks, they predominantly function as black boxes. This lack of transparency hinders trust and makes it difficult to understand the basis for their predictions. Further, these models typically require large labeled datasets for training and may generalize poorly to manipulation techniques that differ significantly from their training data.

\noindent\textbf{Positioning of Our Approach.}
Our work is motivated by the gap between interpretability and performance in existing image forensic methods. Traditional algorithms often expose meaningful forensic signals but are difficult to combine into an overall assessment, while learning-based methods can achieve strong detection performance but often provide limited interpretability. We address this gap through FRAME, which integrates evidence from multiple forensic algorithms, supports adaptive path selection, and preserves both fused outputs and method-specific contributions for analysis. In this paper, the emphasis is on the method formulation, evidence-fusion strategy, and a theoretical account of when adaptive path selection is beneficial. This places our approach at the intersection of image forensics, evidence fusion, and interpretable forensic analysis.

\section{Methodology}
\label{sec:methodology}

FRAME models image forensic analysis as adaptive path selection over a set of modular forensic algorithms. Given an input image \(I\) and a suspected manipulation type \(M_t\), the method samples candidate analysis paths from a shared forensic module pool, scores these paths with a learned selector, and fuses the outputs of the selected paths into a final prediction. Figure~\ref{fig:forensicfusion_method_illustration} gives an overview of this pipeline, and Algorithm~\ref{alg:frame} summarizes the offline training and online inference phases. For clarity, the theoretical development in Section~\ref{sec:theory} focuses on the core top-1 path-selection problem, while the top-\(k\) fusion variant used in the experiments is a practical extension of the same pipeline.

\begin{figure*}[t]
  \centering
  \includegraphics[width=\textwidth]{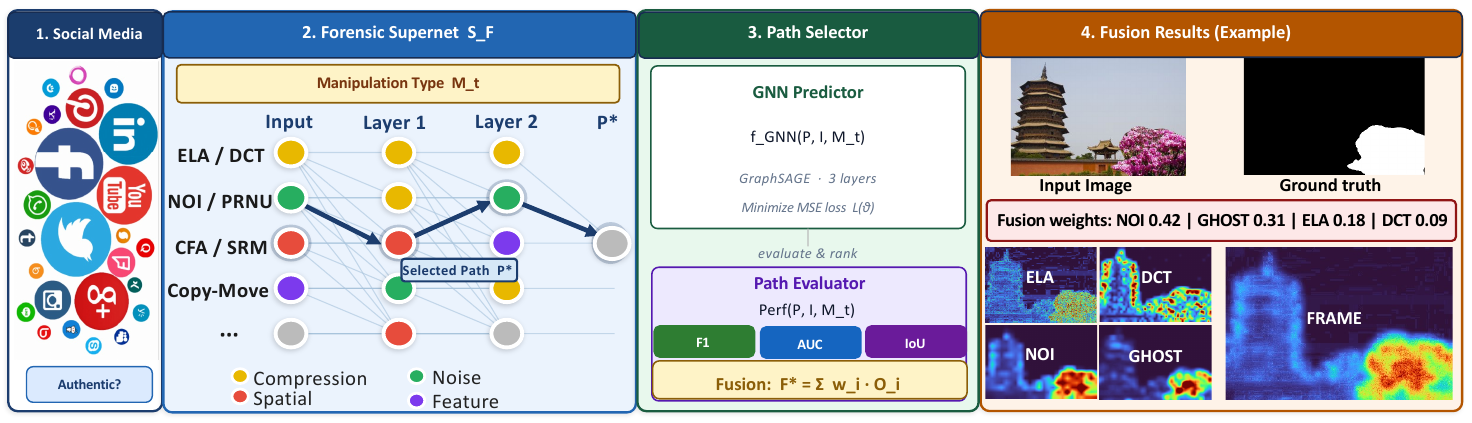}
  \caption{Overview of FRAME. Given an input image, FRAME samples candidate paths from a forensic supernet, scores them with a learned path selector, and fuses the selected outputs into a final heatmap $F^*$.}
  \label{fig:forensicfusion_method_illustration}
\end{figure*}

\noindent\textbf{Forensic algorithm module.}
Each forensic algorithm module \(\mathcal{A}_k\) wraps one underlying algorithm \(A_k\) which maps an input image to a forensic output such as a heatmap, mask, or score. Examples include Error Level Analysis (ELA), Discrete Cosine Transform (DCT) analysis, and noise-based methods, many of which are available in \textit{pyIFD} \cite{vidalmata2023effectiveness}. Formally, \(\mathcal{A}_k = \{alg_k(\cdot), \Theta_k\}\) with \(alg_k(\cdot): \mathcal{I} \rightarrow \mathcal{O}_k\), where \(\Theta_k\) denotes the configurable parameters of the algorithm.

\noindent\textbf{Forensic supernet and analysis path.}
We define the forensic supernet \(\mathcal{S}_{F} = \{\mathbb{A}, \mathcal{C}_{F}\}\), where \(\mathbb{A} = \{\mathcal{A}_1, \dots, \mathcal{A}_N\}\) is the set of available modules and \(\mathcal{C}_{F}\) specifies their allowed combinations. This induces a directed acyclic graph in which nodes correspond to modules and edges represent valid data flow or fusion operations. A forensic analysis path \(\mathcal{P}_j = \{\mathcal{V}_j, \mathcal{E}_j\}\) is a subgraph sampled from \(\mathcal{S}_{F}\), where \(\mathcal{V}_j \subseteq \mathbb{A}\) is the subset of selected modules and \(\mathcal{E}_j\) specifies their connectivity. Each path produces a fused intermediate output \(F_j\) which summarizes the evidence from the selected modules.

\noindent\textbf{Problem formulation.}
Let \(\Pi(\mathcal{S}_{F})\) denote the set of valid paths induced by the supernet. Given an image \(I\) and a suspected manipulation type \(M_t\), the goal is to identify the path \(\mathcal{P}^*\) which maximizes task-specific performance for that image and manipulation context:
\[
\mathcal{P}^* = \argmax_{\mathcal{P} \in \Pi(\mathcal{S}_{F})} \mathrm{Perf}(\mathcal{P}, I, M_t).
\]
In FRAME, this quantity is estimated by a learned predictor \(f_{\mathrm{GNN}}(\mathcal{P}, I, M_t)\) which scores candidate paths before execution \cite{kipf2017semi,lukasik2020neural}. Section~\ref{sec:theory} refines this objective in terms of conditional expected performance and explicit comparison baselines. The practical system retains the top-\(k\) highest-scoring paths and fuses their outputs, which is the variant used in our experiments.

\begin{algorithm}[t]
\small
\DontPrintSemicolon
\caption{FRAME: Adaptive Path Selection and Fusion}
\label{alg:frame}
\KwIn{Supernet $\mathcal{S}_{F}$; eval set $\mathcal{D}_{eval}$; candidate size $K$; fusion size $k$.}
\tcp{Offline training phase}
\ForEach{$I_i \in \mathcal{D}_{eval}$ with manipulation type $M_{t,i}$}{
    Sample $K$ candidate paths $\{\mathcal{P}_1, \dots, \mathcal{P}_K\}$ from $\mathcal{S}_{F}$\;
    \For{$j = 1, \dots, K$}{
        Apply $\mathcal{P}_j$ to $I_i$ and obtain fused output $F_{j,i}$\;
        Compute score $y_{j,i}$ against ground truth (e.g., F1, IoU)\;
    }
}
Build $\mathcal{D}_{train} = \{(\mathcal{P}_j, I_i, M_{t,i}, y_{j,i})\}$\;
Train $f_{\mathrm{GNN}}$ by minimizing $\sum \lVert f_{\mathrm{GNN}}(\mathcal{P}_j, I_i, M_{t,i}) - y_{j,i} \rVert^2$\;
\tcp{Online inference phase}
\KwIn{Test image $I_{\mathrm{new}}$ with type $M_{t,\mathrm{new}}$.}
Sample $K$ candidate paths from $\mathcal{S}_{F}$\;
Score each path: $s_j \gets f_{\mathrm{GNN}}(\mathcal{P}_j, I_{\mathrm{new}}, M_{t,\mathrm{new}})$\;
Select top-$k$ paths $\{\mathcal{P}_{(1)}, \dots, \mathcal{P}_{(k)}\}$ by ranking $\{s_j\}$\;
Execute selected paths to obtain outputs $\{O_{(1)}, \dots, O_{(k)}\}$\;
Fuse: $F^* = \sum_{i=1}^{k} w_i \cdot O_{(i)}$ with learned weights $\{w_i\}$\;
\KwOut{Fused heatmap $F^*$ and detection score.}
\end{algorithm}

\noindent\textbf{GNN-guided path optimization.}
Algorithm~\ref{alg:frame} summarizes the full pipeline. The intuition is simple: different forensic algorithms respond to different artifacts, so the system should not apply all algorithms with equal weight. Instead, it first identifies promising paths, then estimates which of them are likely to perform well, and finally combines the strongest outputs. In the offline phase, we sample $K$ candidate paths from $\mathcal{S}_{F}$ for each training image, optionally guided by heuristics which reflect known strengths of certain algorithms. Each sampled path $\mathcal{P}_k$ is applied to an image $I_i \in \mathcal{D}_{eval}$ to produce a fused output $F_{k,i}$, which is compared with the ground truth using a task-specific metric such as IoU or F1 to obtain a performance score $y_{k,i}$. This process constructs a training set $\mathcal{D}_{train} = \{(\mathcal{P}_k, I_i, M_{t,i}, y_{k,i})\}$, which is used to train a graph neural network predictor $f_{\mathrm{GNN}}$. The predictor takes as input a graph representation of the path, augmented with image-level features and manipulation metadata, and outputs a scalar score. The training objective minimizes the mean squared error between predicted and observed path performance:
\begin{equation}
\mathcal{L}
=
\sum_{(\mathcal{P}_k, I_i, M_{t,i}, y_{k,i}) \in \mathcal{D}_{train}}
\left\lVert
f_{\mathrm{GNN}}(\mathcal{P}_k, I_i, M_{t,i}) - y_{k,i}
\right\rVert^2.
\end{equation}
Once trained, the predictor ranks candidate paths without exhaustively evaluating every path on every new image.

\noindent\textbf{Online inference and fusion.}
At test time, FRAME samples candidate paths from the supernet and scores them using the trained predictor. In the top-1 version, the highest-scoring path is executed to produce the final output $F^*$. In the top-$k$ version, the predictor retains a small set of high-scoring paths, and their outputs are fused into a final heatmap which is then used for both pixel-level localization and image-level detection.

\section{Theoretical Formulation}
\label{sec:theory}

\noindent\textbf{Intuition.}
Different forensic algorithms are reliable in different situations: one may work well on JPEG-compressed splicing, while another may be more sensitive to copy-move forgery. A fixed combination treats all algorithms the same way for every image and wastes the strengths of each. FRAME instead learns a selector that, given an image and its context, predicts which path is most likely to perform well. This section formalizes this idea and identifies conditions under which a sufficiently accurate selector is guaranteed to improve over both uniform fusion and the best single algorithm. For tractability, we focus on the top-1 selection rule; the top-$k$ fusion variant in Section~\ref{sec:experiments} is an empirical extension built on top of this selector.

\noindent\textbf{Notation.}
Let $X \triangleq (I,M,C)\in\mathcal{X}$ denote the context, where $I\in\mathcal{I}$ is an input image, $M\in\mathcal{M}$ is a manipulation type, and $C\in\mathcal{C}$ is auxiliary context. Let $(X,Y)\sim\mathcal{D}$ with supervision $Y\in\mathcal{Y}$ (e.g., $Y\in\{0,1\}^{H\times W}$ for pixel-level masks). Let $\mathcal{A}\triangleq\{A_1,\dots,A_N\}$ be $N$ heterogeneous forensic algorithms. For each $A_i\in\mathcal{A}$, after any necessary resizing and score normalization, define its aligned output $O_i \triangleq A_i(I)\in[0,1]^{H\times W}.$
Let $\mathcal{P}$ be the set of feasible paths induced by a forensic supernet. Each path $P\in\mathcal{P}$ is a DAG $P=(V_P,E_P)$ with $V_P\subseteq\mathcal{A}$ and $E_P\subseteq V_P\times V_P$. Given $\{O_i\}_{A_i\in V_P}$, the fusion rule $\Phi_P$ produces $F_P \triangleq \Phi_P\big(\{O_i\}_{A_i\in V_P}\big)\in\mathcal{F},\
(\text{e.g., }F_P=\sum_{A_i\in V_P} w_{P,i}O_i).$
Let $\ell(F_P,Y)\in[0,1]$ be a bounded loss and define the performance $y(P,X,Y)\triangleq 1-\ell(F_P,Y)\in[0,1].$
The conditional expected performance is $\mu(P,X)\triangleq \mathbb{E}_{Y\sim\mathcal{D}(\cdot\,|\,X)}\big[y(P,X,Y)\big]
= 1-\mathbb{E}_{Y\sim\mathcal{D}(\cdot\,|\,X)}\big[\ell(F_P,Y)\big].$
Throughout, whenever an $\arg\max$ is set-valued, we assume an arbitrary but fixed measurable tie-breaking rule.

\subsection{Problem Setting}
\label{sec:problem-setting}

We study \emph{context-adaptive path selection} over the feasible path set $\mathcal{P}$. A (measurable) selection rule is a mapping $\pi:\mathcal{X}\to\mathcal{P}$, and its expected performance is
$J(\pi)\triangleq \mathbb{E}_{(X,Y)\sim\mathcal{D}}\big[y(\pi(X),X,Y)\big]
=\mathbb{E}_{X\sim\mathcal{D}_X}\big[\mu(\pi(X),X)\big].$
The goal is to find an optimal rule $\pi^\star \in \arg\max_{\pi:\mathcal{X}\to\mathcal{P}} J(\pi),$
which is equivalently characterized by the per-instance oracle
\begin{equation}
\label{eq:global-oracle}
    P^\star(X)\in\arg\max_{P\in\mathcal{P}} \mu(P,X).
\end{equation}

We compare against two canonical baselines that are special elements of $\mathcal{P}$: the single-algorithm path $P^{(i)}$ with $V_{P^{(i)}}=\{A_i\}$ for any $i\in[N]$, and the uniform all-algorithms fusion path $P^{(\mathrm{avg})}$ with
$V_{P^{(\mathrm{avg})}}=\mathcal{A}$ and $w_{P^{(\mathrm{avg})},i}=1/N$. Since $\mu(P,X)$ is unknown, we learn a predictor $f_\theta:\mathcal{P}\times\mathcal{X}\to\mathbb{R}$ from supervised path-level scores using $N_{\mathrm{tr}}$ training contexts. At test time, given $X$, we construct a candidate set $\mathcal{C}(X)\subseteq\mathcal{P}$ and select
\begin{equation}
\label{eq:learned-selection}
    \widehat{P}(X)\in\arg\max_{P\in\mathcal{C}(X)} f_\theta(P,X),
\end{equation}
then output $F_{\widehat{P}(X)}$. For a fair comparison, we assume $\mathcal{C}(X)$ always contains
$\{P^{(i)}\}_{i=1}^N$ and $P^{(\mathrm{avg})}$. The theoretical results below use only the high-probability accuracy property in Assumption~\ref{assump:uniform}.

\subsection{Conditional Theoretical Guarantees}
\label{sec:theory-guarantees}

We provide two \emph{conditional strict improvement} guarantees for the learned selection rule in \eqref{eq:learned-selection}, relative to (i) uniform fusion over all algorithms and (ii) the best single algorithm baseline. Rather than assuming a fixed average margin between the oracle and a baseline, we combine (a) a Tsybakov/no-tie condition \cite{mammen1999smooth,tsybakov2004optimal} controlling how often the best and second-best candidate paths are close, with (b) the existence of a non-negligible subset of contexts on which a given baseline is provably suboptimal. These conditions capture heterogeneous regimes where context-agnostic baselines are systematically limited and a sufficiently accurate scorer can exploit context.

Define the candidate-set oracle
\begin{equation}
\label{eq:cand-oracle}
P^\star_{\mathcal{C}}(X)\in\arg\max_{P\in\mathcal{C}(X)} \mu(P,X).
\end{equation}
Let $\mu_{(1)}(X)\ge \mu_{(2)}(X)$ denote the largest and the second-largest values in the finite set
$\{\mu(P,X):P\in\mathcal{C}(X)\}$ (ties allowed), and define the candidate-set gap
\begin{equation}
\label{eq:gap-def}
g(X)\triangleq \mu_{(1)}(X)-\mu_{(2)}(X)\in[0,1].
\end{equation}

\begin{assumption}[Bounded loss]
\label{assump:bounded}
For any $P\in\mathcal{P}$ and $(X,Y)\sim\mathcal{D}$, $\ell(F_P,Y)\in[0,1]$ almost surely.
Consequently, $y(P,X,Y)=1-\ell(F_P,Y)\in[0,1]$ and $\mu(P,X)\in[0,1]$.
\end{assumption}

\begin{assumption}[Finite candidate set with baseline inclusion]
\label{assump:candidate}
For each $X\in\mathcal{X}$, the candidate set $\mathcal{C}(X)\subseteq\mathcal{P}$ is finite and non-empty, and $\{P^{(i)}\}_{i=1}^N \cup \{P^{(\mathrm{avg})}\}\subseteq \mathcal{C}(X),\ \forall X\in\mathcal{X}.$
\end{assumption}

\begin{assumption}[High-probability uniform accuracy on the candidate set]
\label{assump:uniform}
For any $\delta\in(0,1)$, with probability at least $1-\delta$ over the training data,
\begin{equation}
\label{eq:hp-uniform}
\sup_{X\in\mathcal{X}}\ \sup_{P\in\mathcal{C}(X)} \big|f_\theta(P,X)-\mu(P,X)\big|\ \le\ \varepsilon(N_{\mathrm{tr}},\delta),
\end{equation}
for some error function $\varepsilon(N_{\mathrm{tr}},\delta)\ge 0$. This is the only property of $f_\theta$ used in the results below.
\end{assumption}

\begin{assumption}[Tsybakov/no-tie condition on the candidate-set gap]
\label{assump:tsybakov}
There exist constants $c_0>0$ and $\alpha>0$ such that for all $t\in[0,1]$, $\mathbb{P}\big(g(X)\le t\big)\ \le\ c_0\, t^\alpha,$
where $g(X)$ is defined in \eqref{eq:gap-def} and the probability is over $X\sim\mathcal{D}_X$.
\end{assumption}

\begin{assumption}[Uniform-all baseline is suboptimal on a non-negligible subset]
\label{assump:avg-subopt}
There exist a measurable set $\mathcal{S}_{\mathrm{avg}}\subseteq\mathcal{X}$ and constants
$p_{\mathrm{avg}}\in(0,1]$ and $\gamma_{\mathrm{avg}}>0$ such that
$\mathbb{P}(X\in\mathcal{S}_{\mathrm{avg}})\ge p_{\mathrm{avg}}$ and $\mu(P^\star_{\mathcal{C}}(X),X)\ \ge\ \mu(P^{(\mathrm{avg})},X)\ +\ \gamma_{\mathrm{avg}},
\ \forall X\in\mathcal{S}_{\mathrm{avg}}.$
\end{assumption}

\noindent\emph{Remark.}
Assumption~\ref{assump:avg-subopt} captures the regime where averaging all algorithms can dilute specialized signals. On a subset of contexts with probability at least $p_{\mathrm{avg}}$, some candidate path outperforms uniform fusion by a fixed margin $\gamma_{\mathrm{avg}}$.

\begin{assumption}[Best-single baseline is suboptimal on a non-negligible subset]
\label{assump:single-subopt}
Let $i^\dagger\in\arg\max_{i\in[N]}\mathbb{E}_{X\sim\mathcal{D}_X}[\mu(P^{(i)},X)]$ and define the
best-single baseline $P^{(\mathrm{bs})}\triangleq P^{(i^\dagger)}$.
There exist a measurable set $\mathcal{S}_{\mathrm{bs}}\subseteq\mathcal{X}$ and constants
$p_{\mathrm{bs}}\in(0,1]$ and $\gamma_{\mathrm{bs}}>0$ such that
$\mathbb{P}(X\in\mathcal{S}_{\mathrm{bs}})\ge p_{\mathrm{bs}}$ and $\mu(P^\star_{\mathcal{C}}(X),X)\ \ge\ \mu(P^{(\mathrm{bs})},X)\ +\ \gamma_{\mathrm{bs}},
\ \forall X\in\mathcal{S}_{\mathrm{bs}}.$
\end{assumption}

\noindent\emph{Remark.}
Assumption~\ref{assump:single-subopt} formalizes heterogeneity: the best fixed single algorithm in expectation need not be optimal for every context. Accordingly, there exists a subset of contexts of mass at least $p_{\mathrm{bs}}$ on which another candidate path outperforms that fixed choice by margin $\gamma_{\mathrm{bs}}$.

\begin{theorem}[Strict improvement over uniform-all fusion]
\label{thm:strict-avg}
Assume Assumptions~\ref{assump:bounded}--\ref{assump:uniform}, Assumption~\ref{assump:tsybakov}, and
Assumption~\ref{assump:avg-subopt}. Let $\widehat{P}(X)$ be the learned selection rule in \eqref{eq:learned-selection},
and let $P^{(\mathrm{avg})}$ denote the uniform-all fusion baseline. Let $\delta\in(0,1)$ be the failure probability in
Assumption~\ref{assump:uniform}. Let $c_0>0$ and $\alpha>0$ be the no-tie constants in
Assumption~\ref{assump:tsybakov}, and define $C_\alpha \triangleq 2^{\alpha+1}c_0$. If
\begin{equation}
\label{eq:cond-strict-avg}
\varepsilon(N_{\mathrm{tr}},\delta)\ <\ \left(\frac{p_{\mathrm{avg}}\gamma_{\mathrm{avg}}}{C_\alpha}\right)^{\!\frac{1}{\alpha+1}},
\end{equation}
then with probability at least $1-\delta$ (over the training data),
\[
\mathbb{E}_{X\sim\mathcal{D}_X}\!\big[\mu(\widehat{P}(X),X)\big]
\ >\
\mathbb{E}_{X\sim\mathcal{D}_X}\!\big[\mu(P^{(\mathrm{avg})},X)\big].
\]
\end{theorem}

\noindent This theorem formalizes a simple conditional message: on the high-probability event in
Assumption~\ref{assump:uniform}, if the scoring error $\varepsilon(N_{\mathrm{tr}},\delta)$ is small enough and
Assumption~\ref{assump:avg-subopt} holds, then selecting the best-scored path improves over the ``average over all
algorithms'' baseline in expected performance. Thus, under heterogeneous contexts where uniform fusion is suboptimal on
a non-negligible subset, sufficiently accurate scoring leads to a strict average-performance gain. Detailed proof is
provided in Appendix~\ref{sec:app-proof-thm-avg}.

\begin{theorem}[Strict improvement over the best single algorithm]
\label{thm:strict-bs}
Assume Assumptions~\ref{assump:bounded}--\ref{assump:uniform}, Assumption~\ref{assump:tsybakov}, and
Assumption~\ref{assump:single-subopt}. Let $\widehat{P}(X)$ be the learned selection rule in \eqref{eq:learned-selection}.
Let $\delta$ be as above, and let $c_0,\alpha$ and $C_\alpha\triangleq 2^{\alpha+1}c_0$ be as above. If
\begin{equation}
\label{eq:cond-strict-bs}
\varepsilon(N_{\mathrm{tr}},\delta)\ <\ \left(\frac{p_{\mathrm{bs}}\gamma_{\mathrm{bs}}}{C_\alpha}\right)^{\!\frac{1}{\alpha+1}},
\end{equation}
then with probability at least $1-\delta$ (over the training data),
\[
\mathbb{E}_{X\sim\mathcal{D}_X}\!\big[\mu(\widehat{P}(X),X)\big]
\ >\
\max_{i\in[N]}\mathbb{E}_{X\sim\mathcal{D}_X}\!\big[\mu(P^{(i)},X)\big].
\]
\end{theorem}

\noindent This theorem shows that the benefit is not limited to beating averaging. On the same high-probability event
from Assumption~\ref{assump:uniform}, if the data contain heterogeneous regimes where the best fixed single algorithm
is systematically dominated on a non-negligible subset, then a sufficiently accurate context-dependent selector also
improves over that best fixed single choice in expectation. Detailed proof is provided in
Appendix~\ref{sec:app-proof-thm-bs}.

\noindent\textbf{Connection to the experiments.}
The baselines emphasized above map directly to the empirical protocol: Section~\ref{sec:experiments} compares against the uniform-all handcrafted fusion baseline \(P^{(\mathrm{avg})}\) and against single-algorithm handcrafted baselines \(\{P^{(i)}\}\), while the ablation study tests whether moving from top-1 selection to top-\(k\) fusion yields additional gains beyond the core selection principle formalized here. The present theory therefore supports the selector-level comparison, while the extra top-\(k\) fusion gains should be interpreted as an empirical extension beyond the scope of the formal guarantees.
\section{Experiments}
\label{sec:experiments}

We evaluate FRAME with four research questions. \textbf{RQ1:} Does adaptive selection and fusion improve over handcrafted forensic baselines that are built from the same family of forensic cues? \textbf{RQ2:} Does FRAME remain competitive with recent deep learning baselines when all methods are evaluated on the same external test sets? \textbf{RQ3:} Which parts of the method are responsible for the final performance gain, especially learned selection and learned fusion? \textbf{RQ4:} Is the method stable with respect to its main hyperparameters, and does the added accuracy come with acceptable computational cost? The rest of this section answers these questions in order. Detailed dataset information, implementation settings, computation cost results, and additional qualitative comparisons are provided in the appendix.

\subsection{Experimental Setup}

We train the learned parts of FRAME only on CASIA v2~\cite{dong2013casia}. We use 8{,}831 images for training and 1{,}918 images for validation and early stopping. We evaluate all methods on four external test sets. CASIA v1~\cite{dong2013casia} is the main benchmark and contains 1{,}754 images. Coverage~\cite{wen2016coverage} contains 200 images and is used to test copy-move forgeries. Columbia~\cite{hsu2006detecting} contains 363 images and is used only for image-level detection under our protocol. RealisticTampering~\cite{korus2017multiscale} contains 440 images and is used to test realistic splicing across cameras. This protocol separates training from evaluation and allows us to measure generalization beyond the training distribution. A concise summary of the dataset protocol is given in Appendix Table~\ref{tab:appendix_dataset_protocol}.

FRAME is built on a pool of handcrafted forensic algorithms from pyIFD~\cite{vidalmata2023effectiveness}. These algorithms do not require training. The learned part of the system consists of a path selector and a fusion module, both of which are trained only on the CASIA v2 training split. The selector uses a lightweight GraphSAGE-style architecture with three layers~\cite{hamilton2017inductive}. Unless otherwise noted, we sample \(K=50\) candidate paths and fuse the top \(k=5\) selected paths in the main experiments. We compare against three groups of baselines. The first group contains handcrafted baselines that use the same pyIFD pool without learned selection. The second group contains two shallow learned ensemble methods, a Random Forest (RF-Ensemble) and a gradient-boosted classifier (XGB-Ensemble), both trained directly on the concatenated pyIFD module outputs. The third group contains four representative deep learning methods, namely TruFor~\cite{guillaro2023trufor}, MMFusion~\cite{triaridis2024exploring}, ManTraNet~\cite{wu2019mantra}, and CAT-Net~\cite{kwon2021catnet}. We use official pretrained models for the deep learning baselines whenever available and follow their official inference procedure as closely as possible. A complete configuration table is given in Appendix Table~\ref{tab:appendix_config}.

We report both localization and detection results. For datasets with localization masks, we report pixel-level F1 and mean intersection over union. For image-level detection, we report AUC and accuracy when appropriate. Since Columbia is used only for detection, we report only image-level results on that dataset. We also measure efficiency, including average inference time per image, peak GPU memory, peak CPU memory, model size, and storage cost. For FRAME, we further separate the runtime into module execution time, selector time, and fusion time. The full computation-cost comparison is in Appendix Table~\ref{tab:appendix_compute_cost}.

\subsection{Comparison with Handcrafted Baselines}

To answer RQ1, we compare FRAME with handcrafted baselines that are built from the same family of forensic cues. This comparison is important because FRAME uses the same underlying forensic pool. A positive result would show that the gain comes from the proposed selection and fusion strategy rather than from introducing a different feature source. We therefore compare our method with the best single pyIFD module, uniform fusion over all available pyIFD modules, random selection with uniform fusion, and heuristic selection with uniform fusion. We further include RF-Ensemble and XGB-Ensemble, which learn a fixed global combination of the same forensic signals. This design tests whether a learned path policy provides an advantage over both fixed combinations and learned but context-agnostic combinations.

Table~\ref{tab:main_results} reports the main quantitative results. The handcrafted rows show a clear pattern. The best single pyIFD module is stronger than the simple fusion baselines on all four datasets, which indicates that naive aggregation can dilute useful forensic evidence. Heuristic selection is the strongest non-learned baseline, but it remains well below FRAME on every test set. RF-Ensemble and XGB-Ensemble both outperform heuristic selection, confirming that learning a combination of forensic signals is beneficial. XGB-Ensemble is slightly stronger than RF-Ensemble across all datasets. However, FRAME improves over XGB-Ensemble by 0.067 in detection AUC, 0.079 in F1, and 0.065 in mIoU on CASIA v1, with consistent gaps on the other test sets. The margin between RF-Ensemble and XGB-Ensemble is small, while the gap from XGB-Ensemble to FRAME is substantially larger. This indicates that the main source of improvement is per-image adaptive routing rather than simply learning a stronger fixed combination of forensic signals. These results support RQ1 and show that adaptive selection and fusion are more effective than both fixed and globally learned combinations of the same handcrafted cues.

\subsection{Comparison with Deep Learning Baselines}

To answer RQ2, we compare FRAME with recent deep learning methods. This question is important because modern baselines often rely on larger training sets and can learn complex visual features. At the same time, these methods may have higher training and inference cost and may not generalize uniformly across manipulation types. To make this comparison fair, we evaluate all methods on the same four external test sets and use official pretrained checkpoints whenever available. The deep learning baselines in our final benchmark are TruFor~\cite{guillaro2023trufor}, MMFusion~\cite{triaridis2024exploring}, ManTraNet~\cite{wu2019mantra}, and CAT-Net~\cite{kwon2021catnet}.

Table~\ref{tab:main_results} also answers RQ2. FRAME achieves the best detection AUC on CASIA v1, Coverage, and RealisticTampering, and it achieves the best localization scores on all three datasets with pixel-level masks. The strongest deep baseline is TruFor, which is close to FRAME on CASIA v1 and Coverage and remains the best method on Columbia, where it reaches 0.924 AUC and 0.891 accuracy. FRAME is slightly below TruFor on this detection-only benchmark, but it is stronger on the three datasets that also require localization. In particular, FRAME improves over TruFor by 0.017 in detection AUC on CASIA v1, 0.036 on Coverage, and 0.025 on RealisticTampering. The corresponding gains in localization are also consistent. On CASIA v1, FRAME improves over TruFor by 0.013 in F1 and 0.014 in mIoU. On Coverage, the gains are 0.021 in F1 and 0.016 in mIoU. On RealisticTampering, the gains are 0.028 in F1 and 0.016 in mIoU. These results suggest that the modular design of FRAME remains competitive with end-to-end learned detectors and is especially effective when localization quality matters. Additional qualitative comparisons are shown in Appendix Figure~\ref{fig:appendix_qualitative}.

All deep learning baselines are evaluated using their official pretrained checkpoints without fine-tuning on CASIA v2. This zero-shot protocol reflects a realistic deployment scenario and avoids confounds from architecture-specific fine-tuning differences. We note that these baselines were trained on substantially larger and more diverse datasets, whereas FRAME trains only its lightweight selector and fusion module (approximately 44{,}000 parameters) on the CASIA v2 split.

\begin{table*}[t]
\centering
\caption{Quantitative comparison on four test benchmarks. Columbia is detection-only. \textbf{Bold}: best per column. Subscripts: $\pm$ std over three evaluation runs.}
\label{tab:main_results}
\small
\setlength{\tabcolsep}{3pt}
\resizebox{\textwidth}{!}{
\begin{tabular}{lcccccccc}
\toprule
Method
  & \multicolumn{2}{c}{CASIA v1}
  & \multicolumn{2}{c}{Coverage}
  & \multicolumn{2}{c}{Columbia}
  & \multicolumn{2}{c}{RealisticTampering} \\
\cmidrule(lr){2-3}\cmidrule(lr){4-5}\cmidrule(lr){6-7}\cmidrule(lr){8-9}
  & Det.\ AUC & Loc.\ F1 / mIoU
  & Det.\ AUC & Loc.\ F1 / mIoU
  & Det.\ AUC & Acc.
  & Det.\ AUC & Loc.\ F1 / mIoU \\
\midrule
Best single pyIFD
  & $0.612_{\pm.004}$
  & $0.284_{\pm.005}$ / $0.198_{\pm.004}$
  & $0.624_{\pm.005}$
  & $0.318_{\pm.007}$ / $0.237_{\pm.005}$
  & $0.841_{\pm.005}$
  & $0.798_{\pm.006}$
  & $0.598_{\pm.005}$
  & $0.261_{\pm.006}$ / $0.181_{\pm.005}$ \\
Uniform-all pyIFD
  & $0.574_{\pm.002}$
  & $0.251_{\pm.003}$ / $0.172_{\pm.002}$
  & $0.587_{\pm.003}$
  & $0.281_{\pm.004}$ / $0.204_{\pm.003}$
  & $0.817_{\pm.003}$
  & $0.774_{\pm.004}$
  & $0.561_{\pm.003}$
  & $0.224_{\pm.004}$ / $0.152_{\pm.003}$ \\
Random-$K$ + uniform
  & $0.589_{\pm.013}$
  & $0.267_{\pm.016}$ / $0.184_{\pm.012}$
  & $0.601_{\pm.016}$
  & $0.294_{\pm.019}$ / $0.215_{\pm.014}$
  & $0.828_{\pm.015}$
  & $0.782_{\pm.017}$
  & $0.574_{\pm.014}$
  & $0.237_{\pm.017}$ / $0.161_{\pm.013}$ \\
Heuristic-$K$ + uniform
  & $0.628_{\pm.003}$
  & $0.302_{\pm.005}$ / $0.214_{\pm.004}$
  & $0.638_{\pm.004}$
  & $0.332_{\pm.006}$ / $0.248_{\pm.005}$
  & $0.856_{\pm.004}$
  & $0.813_{\pm.006}$
  & $0.612_{\pm.004}$
  & $0.274_{\pm.006}$ / $0.192_{\pm.005}$ \\
RF-Ensemble (pyIFD)
  & $0.658_{\pm.007}$
  & $0.328_{\pm.009}$ / $0.232_{\pm.007}$
  & $0.671_{\pm.008}$
  & $0.356_{\pm.011}$ / $0.268_{\pm.008}$
  & $0.868_{\pm.007}$
  & $0.826_{\pm.008}$
  & $0.643_{\pm.008}$
  & $0.301_{\pm.010}$ / $0.213_{\pm.008}$ \\
XGB-Ensemble (pyIFD)
  & $0.674_{\pm.006}$
  & $0.342_{\pm.008}$ / $0.243_{\pm.006}$
  & $0.683_{\pm.007}$
  & $0.371_{\pm.009}$ / $0.279_{\pm.007}$
  & $0.876_{\pm.006}$
  & $0.834_{\pm.007}$
  & $0.651_{\pm.006}$
  & $0.314_{\pm.008}$ / $0.224_{\pm.006}$ \\
\midrule
TruFor~\cite{guillaro2023trufor}
  & $0.724_{\pm.005}$
  & $0.408_{\pm.007}$ / $0.294_{\pm.006}$
  & $0.718_{\pm.006}$
  & $0.447_{\pm.008}$ / $0.338_{\pm.007}$
  & $\mathbf{0.924}_{\pm.004}$
  & $\mathbf{0.891}_{\pm.005}$
  & $0.687_{\pm.005}$
  & $0.364_{\pm.007}$ / $0.268_{\pm.006}$ \\
MMFusion~\cite{triaridis2024exploring}
  & $0.510_{\pm.006}$
  & $0.152_{\pm.008}$ / $0.104_{\pm.006}$
  & $0.541_{\pm.007}$
  & $0.213_{\pm.009}$ / $0.152_{\pm.007}$
  & $0.863_{\pm.005}$
  & $0.821_{\pm.007}$
  & $0.534_{\pm.006}$
  & $0.187_{\pm.008}$ / $0.128_{\pm.006}$ \\
ManTraNet~\cite{wu2019mantra}
  & $0.651_{\pm.004}$
  & $0.337_{\pm.006}$ / $0.241_{\pm.005}$
  & $0.663_{\pm.005}$
  & $0.378_{\pm.007}$ / $0.281_{\pm.006}$
  & $0.897_{\pm.005}$
  & $0.858_{\pm.006}$
  & $0.638_{\pm.005}$
  & $0.312_{\pm.007}$ / $0.224_{\pm.005}$ \\
CAT-Net~\cite{kwon2021catnet}
  & $0.678_{\pm.005}$
  & $0.368_{\pm.007}$ / $0.263_{\pm.005}$
  & $0.691_{\pm.006}$
  & $0.402_{\pm.008}$ / $0.304_{\pm.006}$
  & $0.916_{\pm.004}$
  & $0.882_{\pm.005}$
  & $0.661_{\pm.005}$
  & $0.338_{\pm.006}$ / $0.244_{\pm.005}$ \\
\midrule
\rowcolor{ourblue}
\textbf{FRAME (ours)}
  & $\mathbf{0.741}_{\pm.009}$
  & $\mathbf{0.421}_{\pm.012}$ / $\mathbf{0.308}_{\pm.010}$
  & $\mathbf{0.754}_{\pm.010}$
  & $\mathbf{0.468}_{\pm.014}$ / $\mathbf{0.354}_{\pm.011}$
  & $0.908_{\pm.009}$
  & $0.867_{\pm.011}$
  & $\mathbf{0.712}_{\pm.010}$
  & $\mathbf{0.392}_{\pm.013}$ / $\mathbf{0.284}_{\pm.010}$ \\
\bottomrule
\end{tabular}}
\end{table*}

\subsection{Ablation of Selection and Fusion}

To answer RQ3, we study which parts of FRAME are responsible for the final performance. The method contains two main ideas. The first idea is learned selection, which chooses candidate paths that fit the input image. The second idea is fusion, which combines the outputs of selected paths into a final response map. To isolate their effects, we compare variants that remove one component at a time. We include a top-1 variant that uses learned selection without fusion, a top-\(k\) variant with uniform fusion, a top-\(k\) variant with softmax-weighted fusion, and the full model with learned fusion. We also keep the simple handcrafted baselines in the table so that the contribution of each component can be interpreted relative to stronger non-learned alternatives.

Table~\ref{tab:ablation_selection_fusion} shows a consistent progression from fixed aggregation to learned selection and learned fusion. Top-1 selection already improves clearly over heuristic selection, which indicates that the selector can identify stronger paths than simple rules. On CASIA v1, top-1 selection improves over heuristic selection by 0.056 in detection AUC, 0.046 in F1, and 0.037 in mIoU. Adding uniform fusion on top of learned selection brings another gain of 0.028 in detection AUC, 0.039 in F1, and 0.030 in mIoU. Replacing uniform fusion with softmax fusion yields a further improvement of 0.012 in detection AUC, 0.016 in F1, and 0.012 in mIoU. The full model with learned fusion performs best and improves over softmax fusion by 0.017 in detection AUC, 0.018 in F1, and 0.015 in mIoU. The full model also improves strongly over the uniform-all pyIFD baseline, with gains of 0.167 in detection AUC, 0.170 in F1, and 0.136 in mIoU. These results support RQ3 and show that both learned selection and learned fusion contribute to the final performance.

\begin{table}[t]
\centering
\scriptsize
\setlength{\tabcolsep}{3.5pt}
\caption{Ablation on selection and fusion strategy evaluated on CASIA v1. Higher is better for all metrics.}
\label{tab:ablation_selection_fusion}
\resizebox{\columnwidth}{!}{%
\begin{tabular}{@{}lcccccc@{}}
\toprule
Method & $K$ & Selection & Fusion & Det.\ AUC & F1 & mIoU \\
\midrule
Uniform-all pyIFD & all & none      & uniform & $0.574$ & $0.251$ & $0.172$ \\
Random-$K$        & 10  & random    & uniform & $0.583$ & $0.259$ & $0.178$ \\
Random-$K$        & 50  & random    & uniform & $0.589$ & $0.267$ & $0.184$ \\
Heuristic-$K$     & 50  & heuristic & uniform & $0.628$ & $0.302$ & $0.214$ \\
\midrule
Top-1 selected              & 50 & learned & none    & $0.684$ & $0.348$ & $0.251$ \\
Top-$k$ + uniform fusion    & 50 & learned & uniform & $0.712$ & $0.387$ & $0.281$ \\
Top-$k$ + softmax fusion    & 50 & learned & softmax & $0.724$ & $0.403$ & $0.293$ \\
Top-$k$ + learned fusion (ours) & 50 & learned & learned & $\mathbf{0.741}$ & $\mathbf{0.421}$ & $\mathbf{0.308}$ \\
\bottomrule
\end{tabular}%
}
\vspace{-2mm}
\end{table}

\subsection{Hyperparameter Sensitivity}

To answer RQ4, we first study whether the method is stable under reasonable hyperparameter changes. We focus on the two main method-specific hyperparameters. The first hyperparameter is the number of sampled candidate paths, denoted by \(K\). The second hyperparameter is the number of selected paths that are fused, denoted by \(k\). These two variables control the balance between exploration, robustness, and efficiency. We vary \(K \in \{5,10,20,50,100\}\) while keeping all other settings fixed. We then fix \(K=50\) and vary \(k \in \{1,3,5,10\}\). This study is intentionally compact because our goal is to show stability rather than to present a large tuning benchmark.

Figure~\ref{fig:sensitivity_combined} summarizes the sensitivity study. The left panel shows the effect of the candidate path size \(K\). Performance improves steadily when \(K\) increases from 5 to 50. Detection AUC rises from 0.698 to 0.741, F1 rises from 0.371 to 0.421, and mIoU rises from 0.264 to 0.308. Increasing \(K\) from 50 to 100 does not bring further gains and leads to a small drop in all three accuracy metrics, while inference time nearly doubles from 9.23 to 18.47 seconds per image. The plateau beyond \(K=50\) suggests that the selector already sees enough diverse candidates at this scale and that additional sampling mainly adds redundancy. This trend suggests that a moderate candidate set is sufficient and that very large candidate pools increase cost more than they improve accuracy. The right panel shows the effect of the fusion cardinality \(k\). Moving from top-1 selection to top-3 and top-5 fusion gives a clear improvement, which indicates that combining several high-quality paths is better than relying on a single path. The best results appear at \(k=5\), where detection AUC reaches 0.741, F1 reaches 0.421, and mIoU reaches 0.308. Increasing to \(k=10\) slightly reduces all three metrics, which suggests that larger fusion sets start to include lower-quality paths. This pattern is also desirable from a practical perspective because it shows that the best operating point does not require overly large fusion sets. A detailed computation cost comparison, which complements this sensitivity study, is reported in Appendix Table~\ref{tab:appendix_compute_cost}.

\begin{figure}[t]
    \centering
    \includegraphics[width=\linewidth]{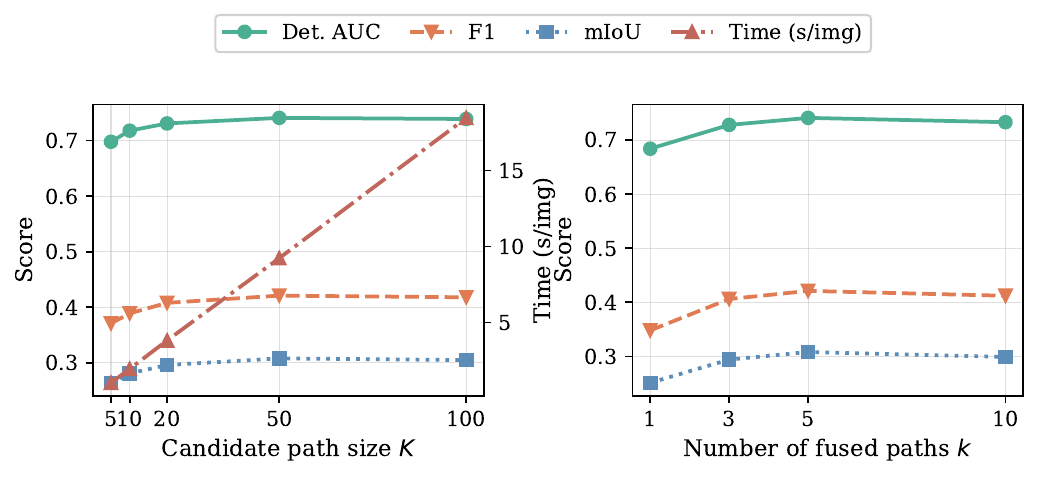}
    \caption{Sensitivity analysis on CASIA v1. The left panel shows the effect of the candidate path size \(K\) with \(k=5\) fixed. Performance improves from \(K=5\) to \(K=50\) and then saturates, while runtime grows steadily. The right panel shows the effect of the number of fused paths \(k\) with \(K=50\) fixed. Moderate fusion performs best, and the peak appears at \(k=5\).}
    \label{fig:sensitivity_combined}
\end{figure}
\section{Limitations and Future Work}
The forensic algorithms in the current module pool are designed to detect artifacts from conventional editing operations such as splicing, copy-move forgery, and JPEG recompression. These methods exploit traces tied to the camera imaging pipeline, which AI-generated or diffusion-inpainted content may not exhibit. However, because FRAME operates over a modular and extensible pool rather than a fixed architecture, it can in principle accommodate forensic modules targeting generative artifacts~\cite{bammey2023synthbuster,frank2020leveraging}.

\section{Conclusion}
We present FRAME, an adaptive framework that organizes diverse forensic algorithms into candidate analysis paths, selects informative paths for each input image, and fuses their outputs into a unified prediction. Our experiments show that learned path selection and fusion consistently outperform both fixed-combination and single-algorithm baselines, demonstrating that adaptive multi-path evidence integration is a promising direction for image forensic analysis.

{\small
\bibliographystyle{ieeenat_fullname}
\bibliography{main}
}

\appendix

\newpage
\section*{Appendix}
\addcontentsline{toc}{section}{Appendix}

\startcontents[sections]
\printcontents[sections]{l}{1}{\setcounter{tocdepth}{2}}

\section{Proofs of Theoretical Guarantees}
\label{sec:appendix-theory}

\subsection{Auxiliary Lemmas}
\label{sec:app-lemmas}

\begin{lemma}[Dominance of the candidate-set oracle]
\label{lem:cand-oracle-dominance}
Under Assumption~\ref{assump:candidate}, for any $X\in\mathcal{X}$ let
\[
P^\star_{\mathcal{C}}(X)\in\arg\max_{P\in\mathcal{C}(X)} \mu(P,X).
\]
Then, for any $P\in\mathcal{C}(X)$,
\[
\mu(P^\star_{\mathcal{C}}(X),X)\ \ge\ \mu(P,X).
\]
\end{lemma}

\begin{proof}
Fix any $X$. By Assumption~\ref{assump:candidate}, $\mathcal{C}(X)$ is finite and non-empty, hence
$\max_{P\in\mathcal{C}(X)} \mu(P,X)$ exists and is attained.
By definition of $\arg\max$, $P^\star_{\mathcal{C}}(X)$ attains this maximum, so
\[
\mu(P^\star_{\mathcal{C}}(X),X)=\max_{P\in\mathcal{C}(X)} \mu(P,X).
\]
Therefore, for any $P\in\mathcal{C}(X)$,
\[
\mu(P^\star_{\mathcal{C}}(X),X)\ge \mu(P,X).
\]
\end{proof}

\begin{lemma}[Approximate maximization under uniform score error]
\label{lem:approx-argmax}
Assume Assumption~\ref{assump:candidate} and that \eqref{eq:hp-uniform} holds for some $\varepsilon$.
Fix any $X\in\mathcal{X}$ and let $\widehat{P}(X)\in\arg\max_{P\in\mathcal{C}(X)} f_\theta(P,X),\
P^\star_{\mathcal{C}}(X)\in\arg\max_{P\in\mathcal{C}(X)} \mu(P,X).$
Then
\[
\mu(\widehat{P}(X),X)\ \ge\ \mu(P^\star_{\mathcal{C}}(X),X)\ -\ 2\varepsilon .
\]
\end{lemma}

\begin{proof}
Fix any $X$ and abbreviate $\widehat{P}=\widehat{P}(X)$ and $P^\star=P^\star_{\mathcal{C}}(X)$.
By optimality of $\widehat{P}$ for $f_\theta$,
\[
f_\theta(\widehat{P},X)\ge f_\theta(P^\star,X).
\]
If \eqref{eq:hp-uniform} holds, then $\mu(\widehat{P},X)\ge f_\theta(\widehat{P},X)-\varepsilon
\qquad\text{and}\qquad
f_\theta(P^\star,X)\ge \mu(P^\star,X)-\varepsilon.$
Combining the three inequalities yields
\[
\mu(\widehat{P},X)\ge \mu(P^\star,X)-2\varepsilon.
\]
\end{proof}

\begin{lemma}[Exact selection when the candidate-set gap is large]
\label{lem:gap-implies-correct}
Assume Assumption~\ref{assump:candidate} and that \eqref{eq:hp-uniform} holds for some $\varepsilon$.
Let $g(X)$ be defined in \eqref{eq:gap-def}.
If $g(X)>2\varepsilon$, then $\widehat{P}(X)$ is a maximizer of $\mu(\cdot,X)$ over $\mathcal{C}(X)$, i.e.,
\[
\mu(\widehat{P}(X),X)\ =\ \mu_{(1)}(X).
\]
\end{lemma}

\begin{proof}
Fix any $X$ and let $P_{(1)}\in\arg\max_{P\in\mathcal{C}(X)}\mu(P,X)$ be a maximizer.
For any $P$ with $\mu(P,X)\le \mu_{(2)}(X)$, \eqref{eq:hp-uniform} implies $f_\theta(P_{(1)},X)\ge \mu_{(1)}(X)-\varepsilon
\qquad\text{and}\qquad
f_\theta(P,X)\le \mu(P,X)+\varepsilon \le \mu_{(2)}(X)+\varepsilon.$
If $g(X)>2\varepsilon$, then
\[
\mu_{(1)}(X)-\varepsilon>\mu_{(2)}(X)+\varepsilon,
\]
hence
\[
f_\theta(P_{(1)},X)>f_\theta(P,X)
\]
for all such $P$.
Thus any maximizer of $f_\theta(\cdot,X)$ must also be a maximizer of $\mu(\cdot,X)$, and therefore
\[
\mu(\widehat{P}(X),X)=\mu_{(1)}(X).
\]
\end{proof}

\begin{lemma}[Expected regret bound under Tsybakov/no-tie]
\label{lem:expected-regret}
Assume Assumptions~\ref{assump:candidate} and~\ref{assump:tsybakov}, and that \eqref{eq:hp-uniform} holds for some $\varepsilon$.
Let
\[
r(X)\triangleq \mu(P^\star_{\mathcal{C}}(X),X)-\mu(\widehat{P}(X),X)\ \ge\ 0.
\]
Then
\[
\mathbb{E}_{X\sim\mathcal{D}_X}[r(X)]\ \le\ C_\alpha\,\varepsilon^{\alpha+1},
\quad\text{where } C_\alpha=2^{\alpha+1}c_0.
\]
\end{lemma}

\begin{proof}
By Lemma~\ref{lem:approx-argmax}, $r(X)\le 2\varepsilon$ for all $X$.
By Lemma~\ref{lem:gap-implies-correct}, $r(X)=0$ whenever $g(X)>2\varepsilon$.
Hence
\[
r(X)\le 2\varepsilon\,\mathbf{1}\{g(X)\le 2\varepsilon\}.
\]
Taking expectation and using Assumption~\ref{assump:tsybakov} with $t=2\varepsilon$ yields
\[
\mathbb{E}[r(X)]
\le 2\varepsilon\cdot \mathbb{P}(g(X)\le 2\varepsilon)
\le 2\varepsilon\cdot c_0(2\varepsilon)^\alpha
= C_\alpha\varepsilon^{\alpha+1}.
\]
\end{proof}

\subsection{Proof of Theorem~\ref{thm:strict-avg}}
\label{sec:app-proof-thm-avg}

\begin{proof}
Let $\mathcal{E}_\delta$ denote the high-probability event in Assumption~\ref{assump:uniform}, i.e.,
\[
\mathcal{E}_\delta
\triangleq
\left\{
\sup_{X\in\mathcal{X}}\sup_{P\in\mathcal{C}(X)}
\big|f_\theta(P,X)-\mu(P,X)\big|
\le \varepsilon(N_{\mathrm{tr}},\delta)
\right\}.
\]
By Assumption~\ref{assump:uniform}, $\mathbb{P}(\mathcal{E}_\delta)\ge 1-\delta$ over the training data. We prove the desired inequality on the event $\mathcal{E}_\delta$.

Throughout the proof, write $\varepsilon=\varepsilon(N_{\mathrm{tr}},\delta)$ and let
\[
r(X)=\mu(P^\star_{\mathcal{C}}(X),X)-\mu(\widehat{P}(X),X).
\]
Then
\begin{equation}
\label{eq:avg-decompose}
\mathbb{E}[\mu(\widehat{P}(X),X)]
\ =\ \mathbb{E}[\mu(P^\star_{\mathcal{C}}(X),X)]\ -\ \mathbb{E}[r(X)].
\end{equation}
By Lemma~\ref{lem:expected-regret}, on the event $\mathcal{E}_\delta$,
\[
\mathbb{E}[r(X)]\le C_\alpha\varepsilon^{\alpha+1}.
\]

By Assumption~\ref{assump:candidate}, $P^{(\mathrm{avg})}\in\mathcal{C}(X)$ for all $X$, hence
Lemma~\ref{lem:cand-oracle-dominance} implies
\[
\mu(P^\star_{\mathcal{C}}(X),X)\ge \mu(P^{(\mathrm{avg})},X)
\qquad\text{for all }X.
\]
Combining this with Assumption~\ref{assump:avg-subopt} yields
\[
\mu(P^\star_{\mathcal{C}}(X),X)-\mu(P^{(\mathrm{avg})},X)
\ \ge\ \gamma_{\mathrm{avg}}\cdot \mathbf{1}\{X\in\mathcal{S}_{\mathrm{avg}}\}.
\]
Taking expectation and using $\mathbb{P}(X\in\mathcal{S}_{\mathrm{avg}})\ge p_{\mathrm{avg}}$ gives
\[
\mathbb{E}[\mu(P^\star_{\mathcal{C}}(X),X)]
\ \ge\ \mathbb{E}[\mu(P^{(\mathrm{avg})},X)]\ +\ p_{\mathrm{avg}}\gamma_{\mathrm{avg}}.
\]
Substituting into \eqref{eq:avg-decompose} yields
\[
\mathbb{E}[\mu(\widehat{P}(X),X)]
\ \ge\ \mathbb{E}[\mu(P^{(\mathrm{avg})},X)]\ +\ p_{\mathrm{avg}}\gamma_{\mathrm{avg}}\ -\ C_\alpha\varepsilon^{\alpha+1}.
\]
Under condition \eqref{eq:cond-strict-avg}, the right-hand side is strictly larger than
$\mathbb{E}[\mu(P^{(\mathrm{avg})},X)]$. Therefore, on $\mathcal{E}_\delta$,
\[
\mathbb{E}_{X\sim\mathcal{D}_X}[\mu(\widehat{P}(X),X)]
>
\mathbb{E}_{X\sim\mathcal{D}_X}[\mu(P^{(\mathrm{avg})},X)].
\]
Since $\mathbb{P}(\mathcal{E}_\delta)\ge 1-\delta$, the theorem follows.
\end{proof}

\subsection{Proof of Theorem~\ref{thm:strict-bs}}
\label{sec:app-proof-thm-bs}

\begin{proof}
Let $\mathcal{E}_\delta$ denote the high-probability event in Assumption~\ref{assump:uniform}, i.e.,
\[
\mathcal{E}_\delta
\triangleq
\left\{
\sup_{X\in\mathcal{X}}\sup_{P\in\mathcal{C}(X)}
\big|f_\theta(P,X)-\mu(P,X)\big|
\le \varepsilon(N_{\mathrm{tr}},\delta)
\right\}.
\]
By Assumption~\ref{assump:uniform}, $\mathbb{P}(\mathcal{E}_\delta)\ge 1-\delta$ over the training data. We prove the desired inequality on the event $\mathcal{E}_\delta$.

Throughout the proof, write $\varepsilon=\varepsilon(N_{\mathrm{tr}},\delta)$ and let
\[
r(X)=\mu(P^\star_{\mathcal{C}}(X),X)-\mu(\widehat{P}(X),X).
\]
As in \eqref{eq:avg-decompose},
\begin{equation}
\label{eq:bs-decompose}
\mathbb{E}[\mu(\widehat{P}(X),X)]
\ =\ \mathbb{E}[\mu(P^\star_{\mathcal{C}}(X),X)]\ -\ \mathbb{E}[r(X)].
\end{equation}
By Lemma~\ref{lem:expected-regret}, on the event $\mathcal{E}_\delta$,
\[
\mathbb{E}[r(X)]\le C_\alpha\varepsilon^{\alpha+1}.
\]

By definition of $i^\dagger$ in Assumption~\ref{assump:single-subopt},
\[
\mathbb{E}[\mu(P^{(\mathrm{bs})},X)]\ =\ \max_{i\in[N]}\mathbb{E}[\mu(P^{(i)},X)].
\]
By Assumption~\ref{assump:candidate}, $P^{(\mathrm{bs})}\in\mathcal{C}(X)$ for all $X$, so
Lemma~\ref{lem:cand-oracle-dominance} yields
\[
\mu(P^\star_{\mathcal{C}}(X),X)\ge \mu(P^{(\mathrm{bs})},X)
\qquad\text{for all }X.
\]
Combining this with Assumption~\ref{assump:single-subopt} gives
\[
\mu(P^\star_{\mathcal{C}}(X),X)-\mu(P^{(\mathrm{bs})},X)
\ \ge\ \gamma_{\mathrm{bs}}\cdot \mathbf{1}\{X\in\mathcal{S}_{\mathrm{bs}}\}.
\]
Taking expectation and using $\mathbb{P}(X\in\mathcal{S}_{\mathrm{bs}})\ge p_{\mathrm{bs}}$ yields
\[
\mathbb{E}[\mu(P^\star_{\mathcal{C}}(X),X)]
\ \ge\ \mathbb{E}[\mu(P^{(\mathrm{bs})},X)]\ +\ p_{\mathrm{bs}}\gamma_{\mathrm{bs}}.
\]
Substituting into \eqref{eq:bs-decompose} yields
\[
\mathbb{E}[\mu(\widehat{P}(X),X)]
\ \ge\ \mathbb{E}[\mu(P^{(\mathrm{bs})},X)]\ +\ p_{\mathrm{bs}}\gamma_{\mathrm{bs}}\ -\ C_\alpha\varepsilon^{\alpha+1}.
\]
Under condition \eqref{eq:cond-strict-bs}, the right-hand side is strictly larger than
$\mathbb{E}[\mu(P^{(\mathrm{bs})},X)]$; since
\[
\mathbb{E}[\mu(P^{(\mathrm{bs})},X)]=\max_{i\in[N]}\mathbb{E}[\mu(P^{(i)},X)],
\]
the claim follows on the event $\mathcal{E}_\delta$. Since $\mathbb{P}(\mathcal{E}_\delta)\ge 1-\delta$, the theorem follows.
\end{proof}

\section{Additional Experimental Details and Results}
\label{sec:appendix_experiments}

This appendix complements Section~\ref{sec:experiments}. Since the main quantitative benchmark table is reported in the main text, this appendix focuses on the dataset protocol, the full implementation details, the computation cost comparison, and additional qualitative examples. Appendix Table~\ref{tab:appendix_dataset_protocol} summarizes the dataset protocol. Appendix Table~\ref{tab:appendix_config} reports the complete configuration of FRAME. Appendix Table~\ref{tab:appendix_compute_cost} reports the resource comparison that complements RQ4. Appendix Figure~\ref{fig:appendix_qualitative} provides additional qualitative comparisons across representative baselines and FRAME.

\subsection{Implementation Details}

We train the learned components of FRAME exclusively on the CASIA v2 dataset, which provides 8{,}831 training images and 1{,}918 validation images. No test-set images are used during training or model selection. All evaluation is performed on four held-out benchmarks: CASIA v1, Coverage, Columbia, and RealisticTampering. Columbia is used only for image-level detection, since we do not use pixel-level localization masks for that dataset under our protocol. This strict train-test separation ensures that all reported results measure generalization to unseen data distributions rather than reuse of benchmark-specific statistics. Table~\ref{tab:appendix_dataset_protocol} summarizes the dataset protocol.

\begin{table}[t]
\centering
\caption{Dataset protocol used in all experiments. The selector and fusion module are trained only on CASIA v2. All methods are evaluated on the same four test sets.}
\label{tab:appendix_dataset_protocol}
\small
\setlength{\tabcolsep}{5pt}
\resizebox{\columnwidth}{!}{%
\begin{tabular}{lccc}
\toprule
Role & Dataset & Size & Usage \\
\midrule
Train & CASIA v2 train & 8{,}831 & Train selector and fusion \\
Val & CASIA v2 val & 1{,}918 & Validation and early stopping \\
Test & CASIA v1 & 1{,}754 & Main benchmark \\
Test & Coverage & 200 & Copy-move benchmark \\
Test & Columbia & 363 & Detection only \\
Test & RealisticTampering & 440 & Realistic splicing benchmark \\
\bottomrule
\end{tabular}%
}
\end{table}

FRAME operates over a fixed pool of handcrafted forensic algorithms from pyIFD~\cite{vidalmata2023effectiveness}. The core pool contains nine modules that cover JPEG compression cues, sensor noise patterns, and camera-related artifacts. We also include six optional modules in the main experiments, which gives a total pool of fifteen modules. All module outputs are normalized heatmaps in $[0,1]$. These outputs are precomputed offline and cached to disk before any model training or evaluation. This design ensures that the learned components operate on fixed forensic evidence and do not repeatedly execute the handcrafted algorithms during training.

A path is an ordered sequence of one to four modules drawn from the module pool. For each image, we sample $K=50$ candidate paths, with up to $\max(10K,100)=600$ attempts to avoid duplicates. Module outputs within a path are fused by uniform averaging to produce a path-level heatmap. The GNN selector then scores all sampled paths, and the top-$k=5$ paths are retained for final fusion. The selector is a lightweight GraphSAGE-style network~\cite{hamilton2017inductive} with three message-passing layers, each with hidden dimension 64. Each module is represented as a graph node with a learnable 64-dimensional embedding. Mean pooling over node embeddings is concatenated with a 9-dimensional image feature vector and a 5-dimensional manipulation-type one-hot vector, and the resulting representation is processed by a small MLP to produce a scalar path-quality score. The selector itself is lightweight and contains approximately 44{,}000 parameters.

The image-level features include log-scaled height and width, mean and standard deviation of grayscale intensity, Shannon entropy, Canny edge density, saturation ratio, and binary indicators for JPEG and PNG formats. The selector is trained by regression, where the target is the pixel-level F1 score of the path-level heatmap against the ground-truth mask. We use Adam with learning rate $10^{-3}$, weight decay $10^{-4}$, batch size 128, and 15 training epochs, and we retain the checkpoint with the lowest validation loss. The fusion module learns a scalar weight for each selected path. These weights are optimized using a combined BCE and Dice loss with equal coefficients. During fusion training, images are downsampled to a maximum side length of 384 pixels. At inference time, the top-$k$ selected heatmaps are combined by learned softmax weighting with temperature $\tau=1.0$. The final image-level detection score is the maximum pixel value of the fused heatmap. Pixel-level localization is computed by thresholding the fused map at 0.5. Table~\ref{tab:appendix_config} summarizes the full configuration.


\begin{table*}[t]
\centering
\caption{Complete experimental configuration of FRAME.}
\label{tab:appendix_config}
\footnotesize
\setlength{\tabcolsep}{4pt}
\renewcommand{\arraystretch}{1.08}
\begin{tabularx}{\textwidth}{@{}
>{\raggedright\arraybackslash}m{2.15cm}
>{\centering\arraybackslash}m{3.00cm}
>{\centering\arraybackslash}m{3.10cm}
>{\raggedright\arraybackslash}X@{}}
\toprule
\textbf{Category} & \textbf{Parameter} & \textbf{Value} & \textbf{Description} \\
\midrule
\multirow{5}{*}{\textbf{Dataset}}
  & Train set             & CASIA v2 & 8{,}831 train / 1{,}918 val images \\
  & Test sets             & \shortstack[c]{CASIA v1, Coverage,\\ Columbia,\\ RealisticTampering} & External benchmarks only \\
  & Mask binarization thr & 0.5 & GT and predicted mask threshold \\
  & Detection score       & max pixel & Max of fused heatmap \\
  & Random seed           & 0 & All sampling and evaluation \\
\midrule
\multirow{4}{*}{\textbf{pyIFD Pool}}
  & Core modules     & 9  & ELA, DCT, NOI1--5, GHOST, BLK, CAGI \\
  & Optional modules & 6  & ADQ1--3, NADQ, CFA1, CAGI\_INV \\
  & Total pool       & 15 & All used in main experiments \\
  & Module output    & heatmap $\in [0,1]$ & Per-pixel anomaly score \\
\midrule
\multirow{5}{*}{\textbf{Path Sampling}}
  & Candidate paths $K$   & 50 & Default; varied in sensitivity study \\
  & Min / max path length & 1 / 4 & Modules per path \\
  & Max graph nodes       & 8 & Supernet nodes per sample \\
  & Max sampling attempts & $\max(10K,\,100)$ & 600 for $K{=}50$ \\
  & Intra-path fusion     & uniform & Equal weight within a path \\
\midrule
\multirow{7}{*}{\textbf{GNN Selector}}
  & Architecture         & \shortstack[c]{GraphSAGE,\\ 3 layers} & Message-passing GNN \\
  & Hidden dimension     & 64 & All GNN layers \\
  & Module embedding dim & 64 & Learnable module ID embedding \\
  & Image feature dim    & 9 & Hand-crafted features \\
  & Manip.\ type dim     & 5 & One-hot encoded \\
  & MLP head             & \shortstack[c]{$78 \to 128 \to$\\ $64 \to 1$} & ReLU + sigmoid output \\
  & Selector parameters  & $\approx$44{,}000 & Checkpoint $<$ 0.2\,MB \\
\midrule
\multirow{5}{*}{\textbf{Image Features}}
  & Spatial   & \shortstack[c]{$\log(1{+}H),$\\ $\log(1{+}W)$} & Log-scaled resolution \\
  & Intensity & \shortstack[c]{mean / 255,\\ std / 255} & Grayscale statistics \\
  & Texture   & \shortstack[c]{entropy / 8,\\ edge density} & Canny (100 / 200) \\
  & Clipping  & saturation ratio & Pixels ${\leq}2$ or ${\geq}253$ \\
  & Format    & is\_JPEG, is\_PNG & Binary codec flags \\
\midrule
\multirow{6}{*}{\textbf{GNN Training}}
  & Optimizer         & \shortstack[c]{Adam\\ ($\beta_1{=}0.9,\,\beta_2{=}0.999$)} & --- \\
  & Learning rate     & $10^{-3}$ & Fixed schedule \\
  & Weight decay      & $10^{-4}$ & L2 regularization \\
  & Batch size        & 128 paths & --- \\
  & Epochs            & 15 & Best val-loss checkpoint saved \\
  & Gradient clipping & 5.0 (global norm) & --- \\
\midrule
\multirow{5}{*}{\textbf{Fusion Training}}
  & Optimizer     & Adam & Same betas as GNN \\
  & Learning rate & $10^{-2}$ & Fixed schedule \\
  & Weight decay  & $10^{-4}$ & --- \\
  & Epochs        & 10 & Fixed schedule \\
  & Loss          & \shortstack[c]{BCE + Dice,\\ $\lambda{=}1.0$ each} & Dice $\epsilon{=}10^{-6}$ \\
\midrule
\multirow{4}{*}{\textbf{Inference}}
  & Fused paths $k$        & 5 & Top-$k$ from GNN; varied in study \\
  & Fusion strategy        & learned weights & Trained on CASIA v2 \\
  & Softmax temperature    & $\tau{=}1.0$ & Score-softmax weighting \\
  & Max image side (train) & 384\,px & Downsampled during fusion training only \\
\midrule
\multirow{3}{*}{\textbf{Hardware}}
  & GPU         & \shortstack[c]{NVIDIA H200\\ (80\,GB)} & GNN / fusion training and inference \\
  & CPU workers & 16 & pyIFD precomputation \\
  & RAM         & 128\,GB & --- \\
\bottomrule
\end{tabularx}
\end{table*}
\subsection{Computational Cost Analysis}

Table~\ref{tab:appendix_compute_cost} reports the resource comparison that complements RQ4. The table shows a clear trade-off between modularity and wall-clock efficiency. The deep learning baselines are substantially faster at inference, which is expected because they execute a single learned model on GPU. In contrast, FRAME evaluates multiple handcrafted forensic modules, many of which are CPU-bound. This design makes the overall runtime much higher than the deep learning baselines. However, this cost should be interpreted together with the modular design of the method. The runtime is dominated by handcrafted module execution rather than by the selector or the learned fusion module.

At the same time, the learned part of FRAME is lightweight. It uses much less GPU memory than the deep learning baselines and maintains a very small learned parameter footprint. The storage cost is also moderate. Compared with uniform-all pyIFD, FRAME increases runtime only modestly while providing much stronger accuracy. This result suggests that the main cost comes from the shared forensic module pool rather than from the learned selection and fusion mechanism. Overall, the numbers in Table~\ref{tab:appendix_compute_cost} show that FRAME trades inference speed for modularity and stronger localization, while keeping the learned component small.

\begin{table*}[t]
\centering
\caption{Computation cost comparison measured on CASIA v1. FRAME runs pyIFD modules on CPU and the GNN selector on GPU. Wall time is dominated by CPU-bound module execution. \(\dagger\): no learned parameters.}
\label{tab:appendix_compute_cost}
\small
\setlength{\tabcolsep}{4pt}
\begin{tabular}{lcccccc}
\toprule
Method & Time (s/img) & GPU Mem (MB) & CPU Mem (MB) & Params (M) & Model Size (MB) & Storage (MB/img) \\
\midrule
Best single pyIFD  & $1.12$  & $0$       & $184$  & $0^\dagger$ & $0^\dagger$ & $0.18$ \\
Uniform-all pyIFD  & $8.42$  & $0$       & $312$  & $0^\dagger$ & $0^\dagger$ & $1.46$ \\
TruFor             & $0.34$  & $2{,}847$ & $891$  & $149.6$     & $599.2$     & $0.12$ \\
MMFusion           & $0.28$  & $3{,}124$ & $746$  & $83.2$      & $332.8$     & $0.09$ \\
ManTraNet          & $0.19$  & $1{,}892$ & $634$  & $38.4$      & $153.6$     & $0.08$ \\
CAT-Net            & $0.41$  & $2{,}156$ & $712$  & $56.7$      & $226.8$     & $0.10$ \\
FRAME (ours)       & $9.23$  & $847$     & $428$  & $0.04$      & $0.2$       & $0.24$ \\
\bottomrule
\end{tabular}
\end{table*}

\subsection{Additional Qualitative Results}

Figure~\ref{fig:appendix_qualitative} provides additional qualitative comparisons between representative deep learning baselines, representative handcrafted cues, and FRAME. Several examples show that DCT and other single handcrafted cues can respond strongly but also produce noisy background activation. TruFor and MMFusion are generally more spatially consistent than the single handcrafted cues, yet their responses may still miss fine boundaries or weaker manipulated regions. FRAME usually produces a more concentrated response in the manipulated area while suppressing part of the spurious background activation that appears in the handcrafted maps.

The figure also shows that the problem remains difficult in some low-texture or small-object cases. In such examples, all methods weaken, and the advantage of FRAME becomes smaller. Still, the qualitative comparisons support the main conclusion of the paper: adaptive selection and learned fusion help combine complementary forensic signals into a cleaner final localization map.

\begin{figure*}[t]
    \centering
    \newcommand{\imgwidth}{0.155\textwidth}
    \begin{tabular}{@{}c@{\hspace{2mm}}c@{\hspace{2mm}}c@{\hspace{2mm}}c@{\hspace{2mm}}c@{\hspace{2mm}}c@{}}
        \small Input Image & \small TruFor & \small MMFusion & \small DCT & \small GHOST & \small Our Method \\
        \includegraphics[width=\imgwidth]{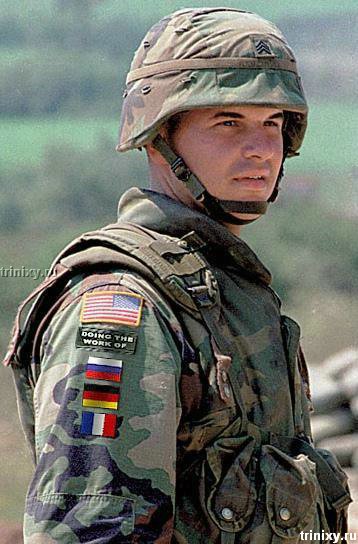} &
        \includegraphics[width=\imgwidth]{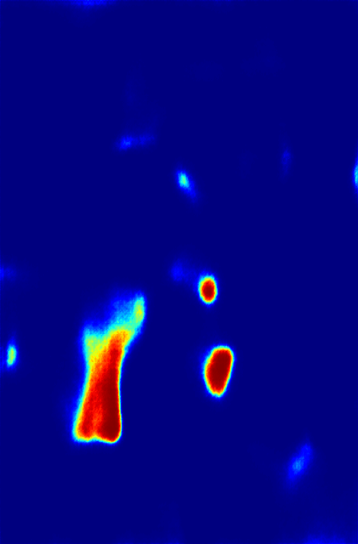} &
        \includegraphics[width=\imgwidth]{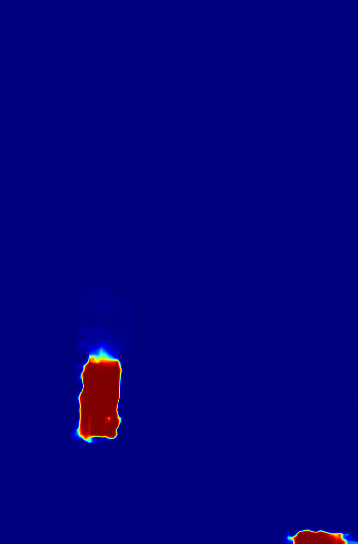} &
        \includegraphics[width=\imgwidth]{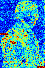} &
        \includegraphics[width=\imgwidth]{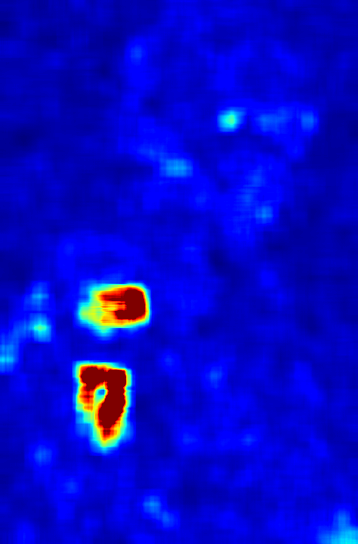} &
        \includegraphics[width=\imgwidth]{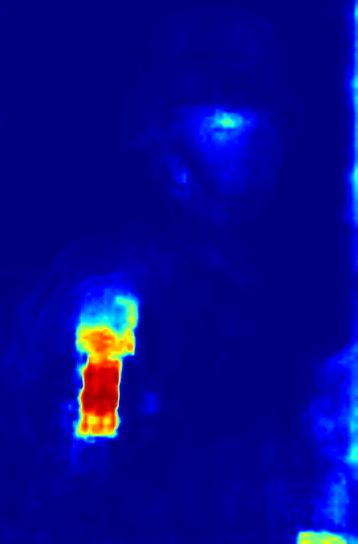} \\
        \includegraphics[width=\imgwidth]{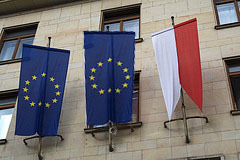} &
        \includegraphics[width=\imgwidth]{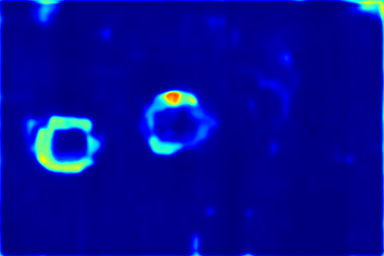} &
        \includegraphics[width=\imgwidth]{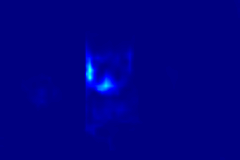} &
        \includegraphics[width=\imgwidth]{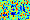} &
        \includegraphics[width=\imgwidth]{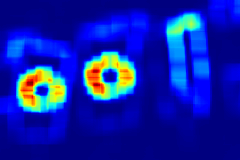} &
        \includegraphics[width=\imgwidth]{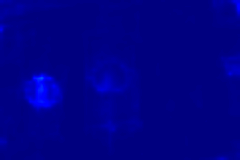} \\
        \includegraphics[width=\imgwidth]{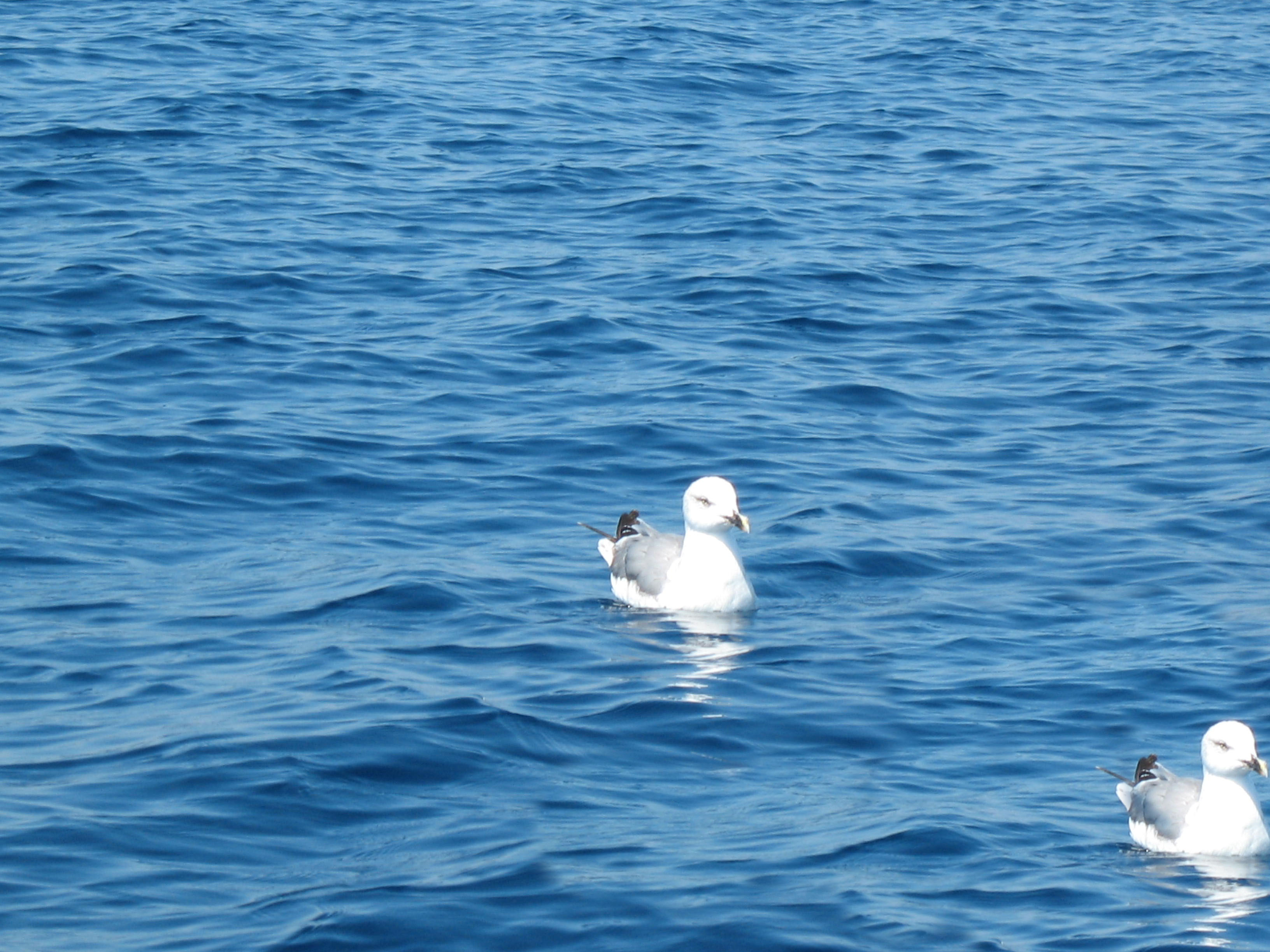} &
        \includegraphics[width=\imgwidth]{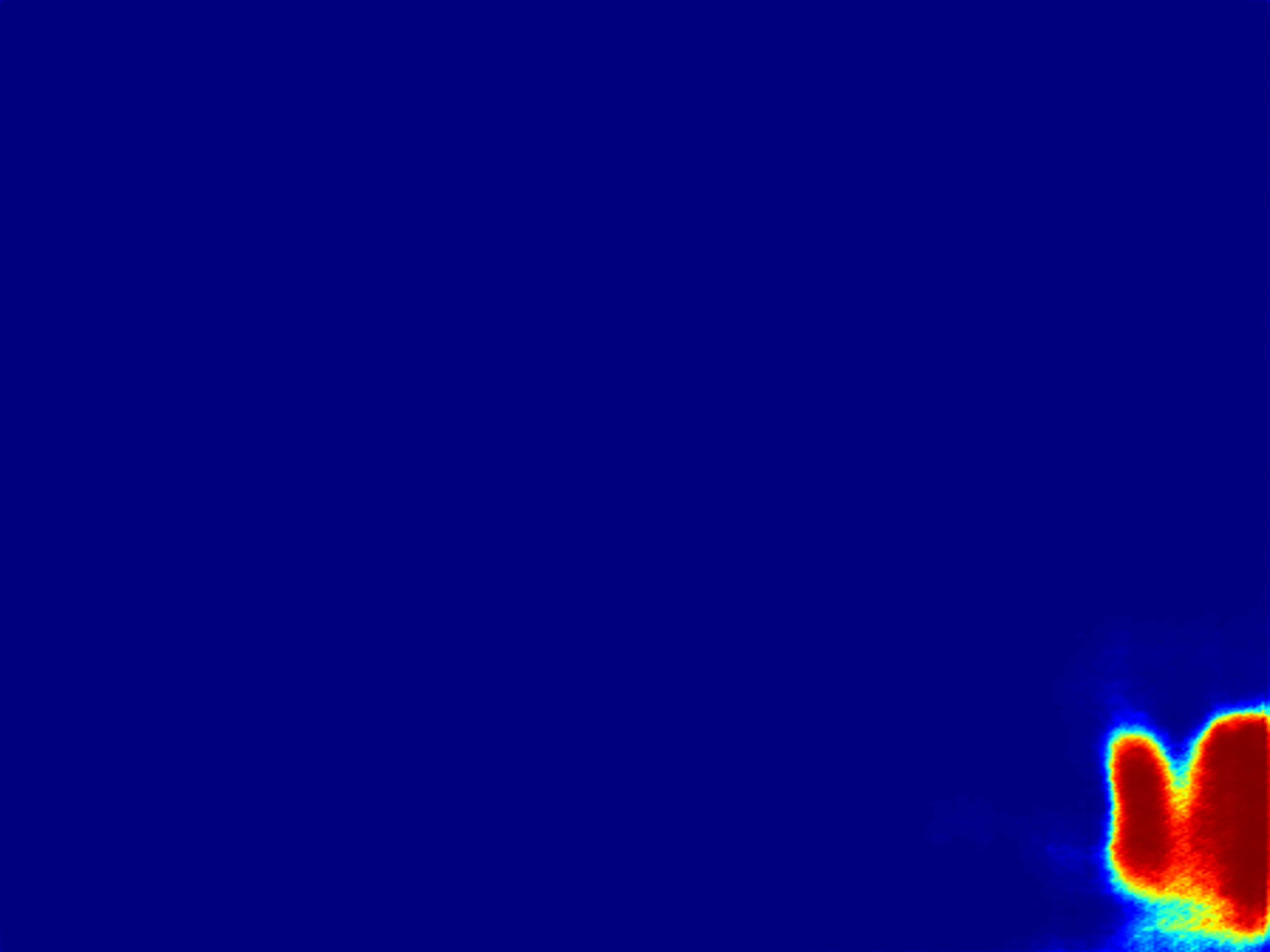} &
        \includegraphics[width=\imgwidth]{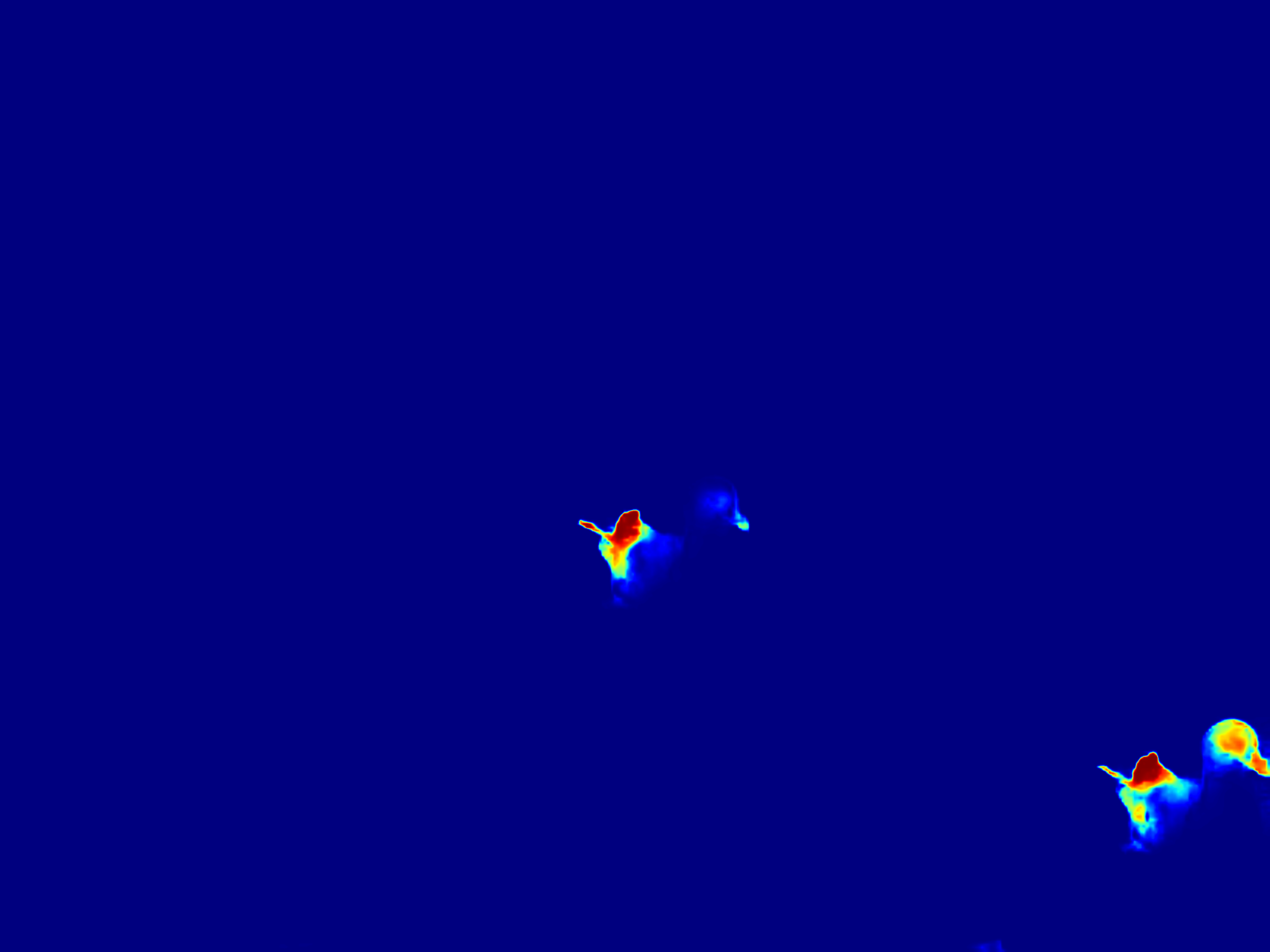} &
        \includegraphics[width=\imgwidth]{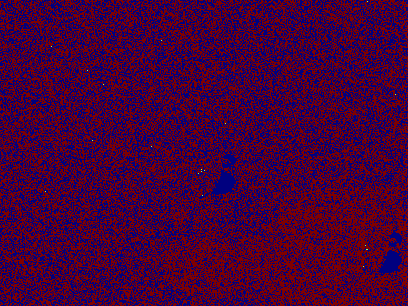} &
        \includegraphics[width=\imgwidth]{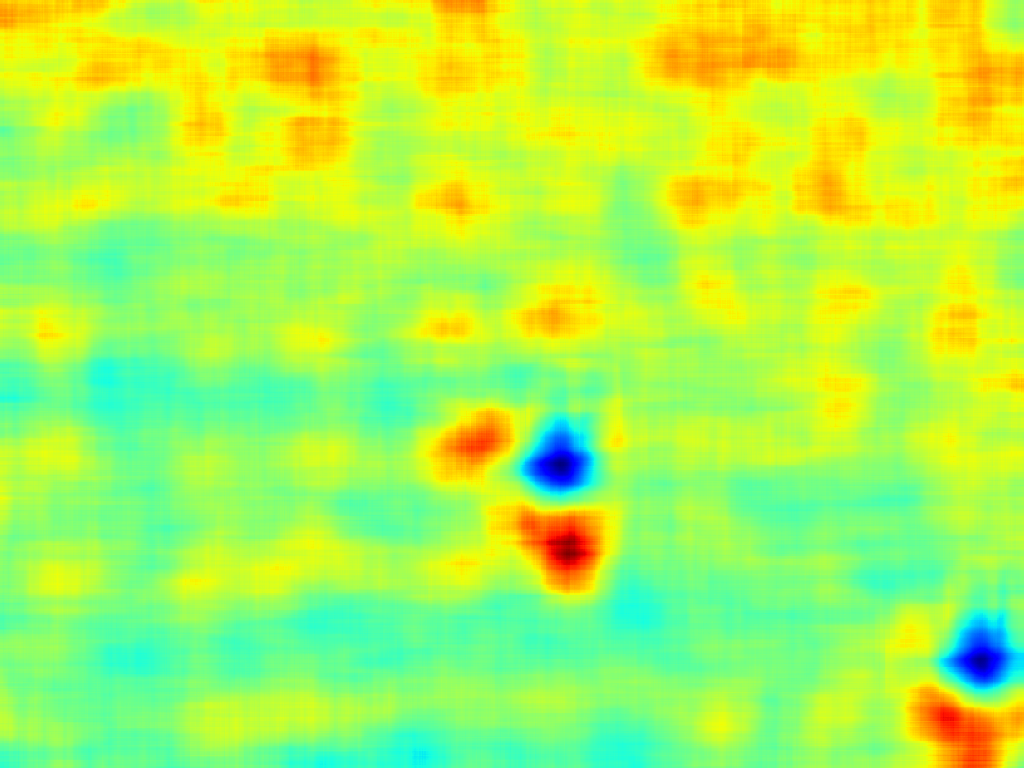} &
        \includegraphics[width=\imgwidth]{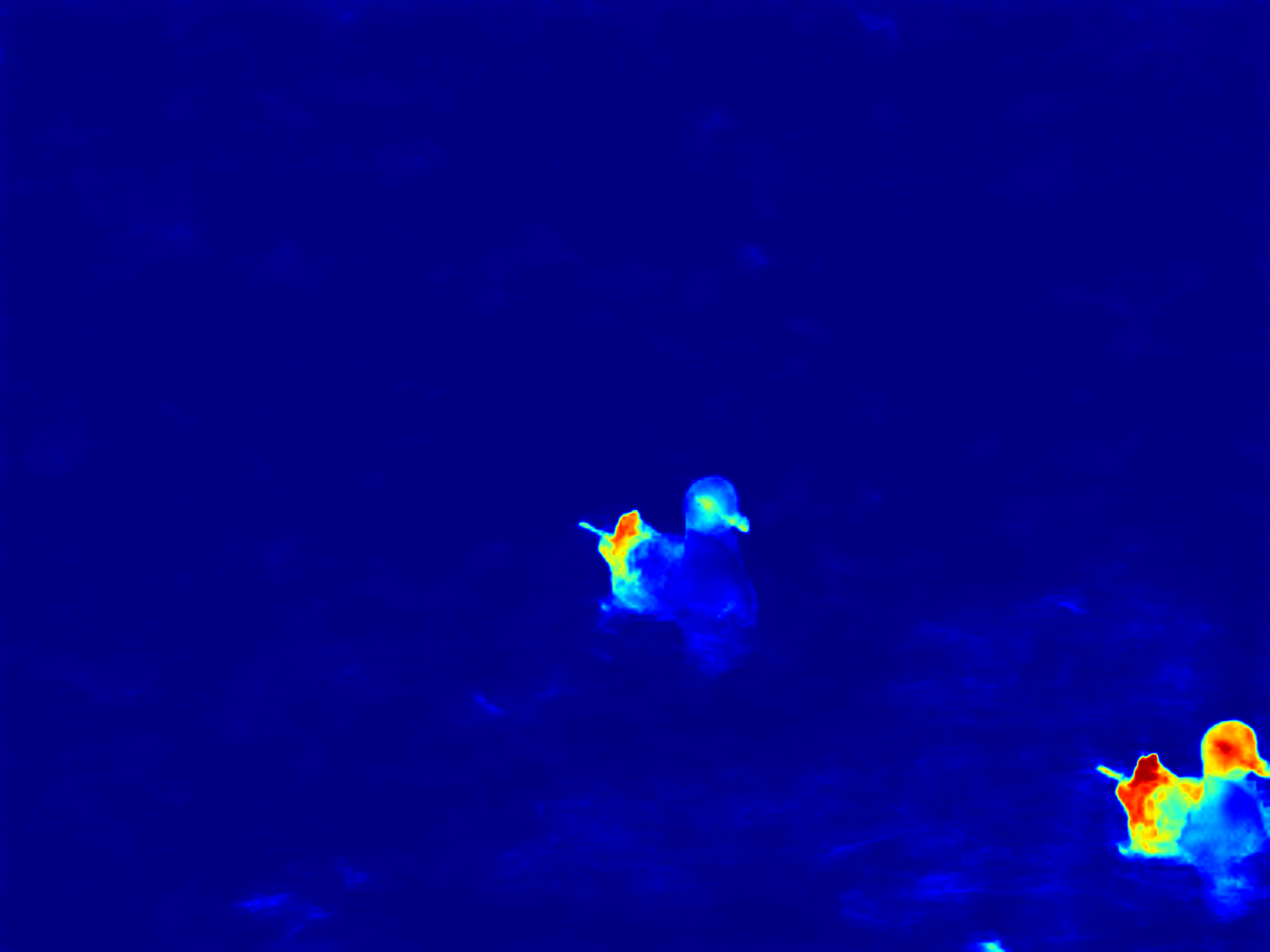} \\
        \includegraphics[width=\imgwidth]{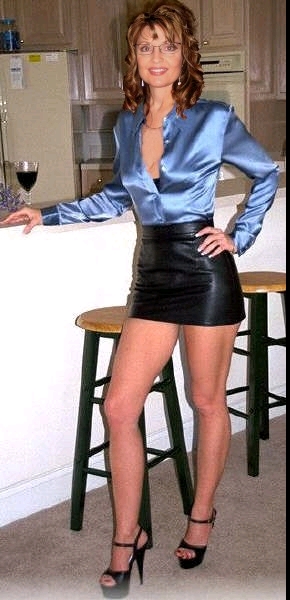} &
        \includegraphics[width=\imgwidth]{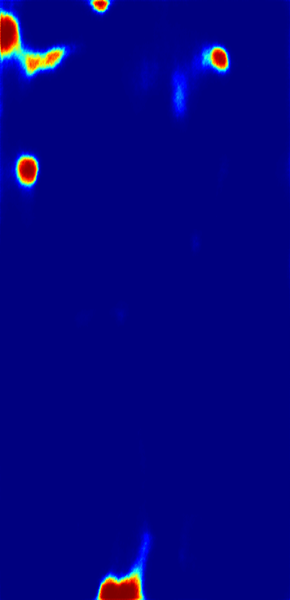} &
        \includegraphics[width=\imgwidth]{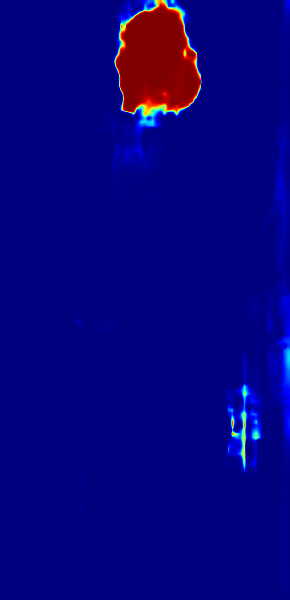} &
        \includegraphics[width=\imgwidth]{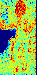} &
        \includegraphics[width=\imgwidth]{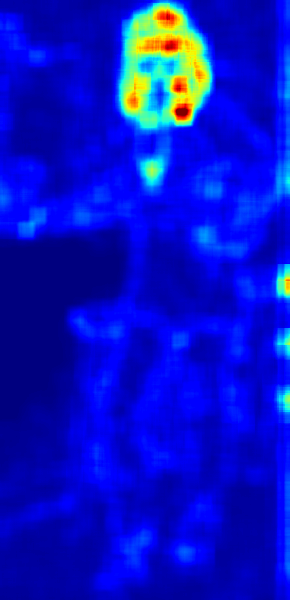} &
        \includegraphics[width=\imgwidth]{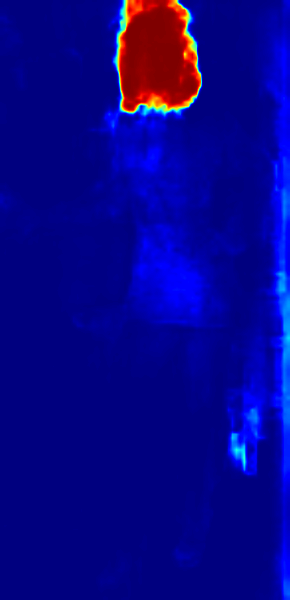} \\
        \includegraphics[width=\imgwidth]{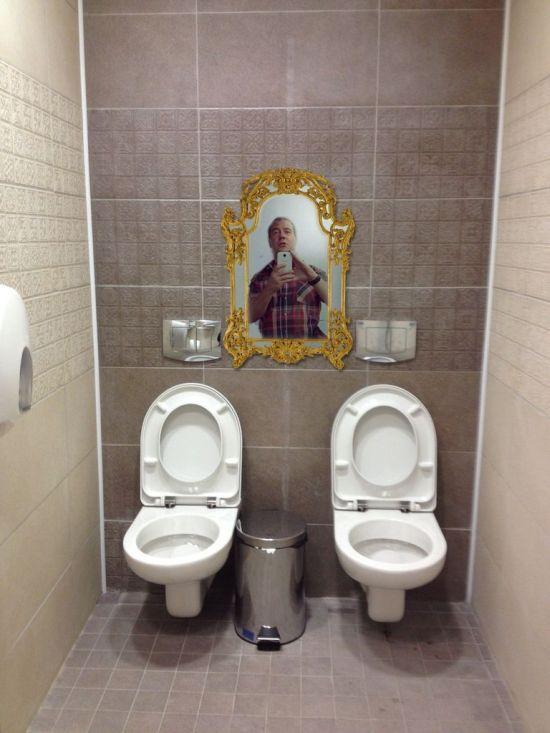} &
        \includegraphics[width=\imgwidth]{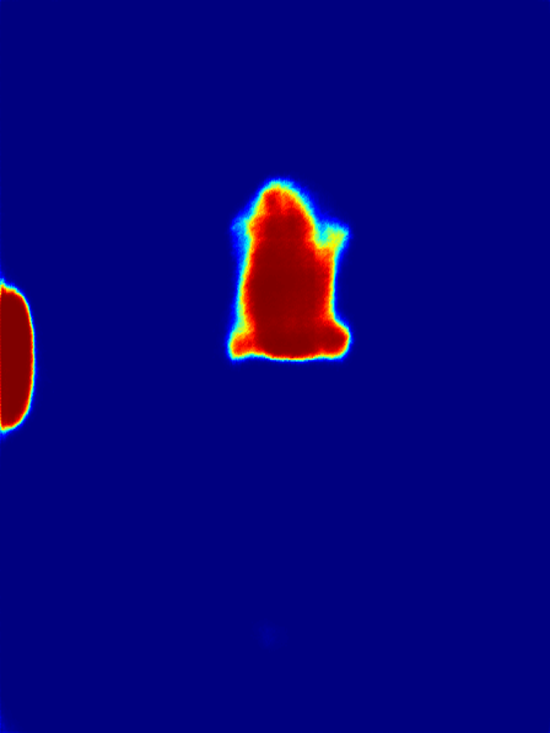} &
        \includegraphics[width=\imgwidth]{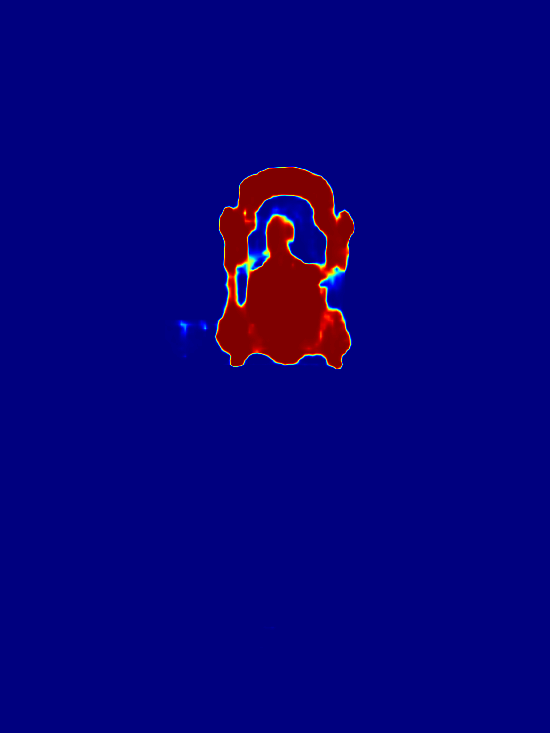} &
        \includegraphics[width=\imgwidth]{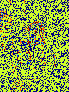} &
        \includegraphics[width=\imgwidth]{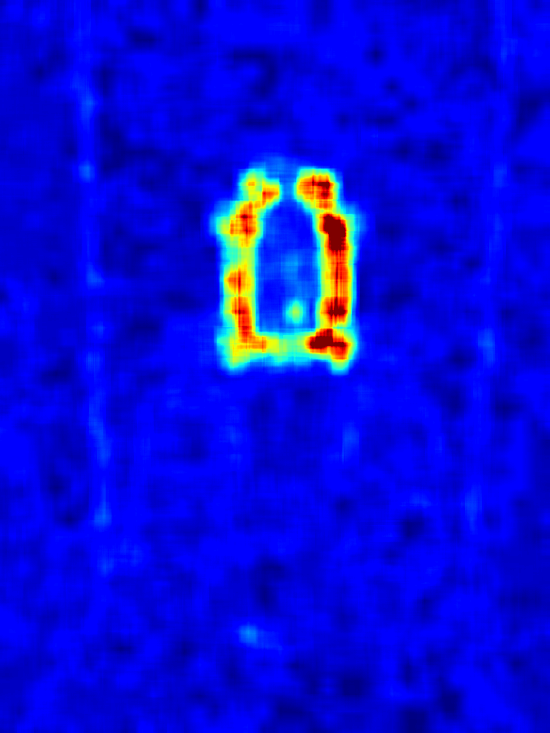} &
        \includegraphics[width=\imgwidth]{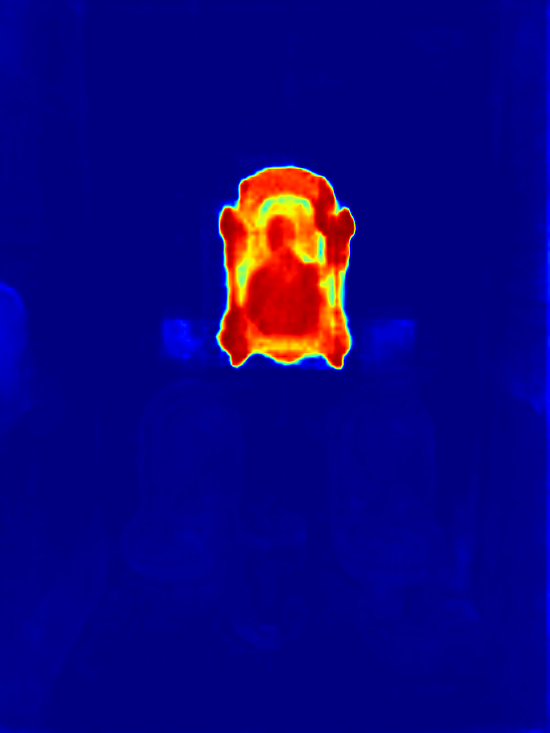} \\
    \end{tabular}
    \caption{Additional qualitative comparison of manipulation localization on five examples. Columns from left to right show the input image, TruFor, MMFusion, DCT, GHOST, and FRAME. The examples illustrate how FRAME combines complementary forensic signals and often produces a cleaner response around the manipulated region.}
    \label{fig:appendix_qualitative}
\end{figure*}

\end{document}